\documentclass{article}

\PassOptionsToPackage{numbers,sort&compress}{natbib}
\usepackage[utf8]{inputenc} 
\usepackage[T1]{fontenc}    
\usepackage{hyperref}       

\usepackage{url}            
\usepackage{booktabs}       
\usepackage{amsfonts}       
\usepackage{nicefrac}       
\usepackage{microtype}      
\usepackage{xcolor}         

\usepackage{amsfonts}[mathscr]
\usepackage{amssymb}
\usepackage{amsmath}
\usepackage{mathtools}
\usepackage{microtype}
\usepackage{graphicx}
\usepackage{natbib}[numbers,sort&compress]
\usepackage{geometry}
\usepackage{booktabs} 
\usepackage{graphicx} 
\usepackage{todonotes}
\usepackage{xcolor}
\usepackage{dsfont}
\usepackage{nicefrac}
\usepackage{tcolorbox}
\usepackage{tabularx}
\usepackage{pifont}
\usepackage{multirow}
\usepackage{algorithm}
\usepackage{algpseudocode}
\usepackage{caption}
\usepackage{wrapfig}
\usepackage[capitalize,noabbrev]{cleveref}
\usepackage{mathrsfs}
\usepackage{algorithm}
\usepackage{algpseudocode}
\usepackage[export]{adjustbox}
\usepackage[list=true]{subcaption}

\usepackage{authblk,textcomp}

\usepackage[cal=euler]{mathalfa}
\usepackage{libertine}

\usepackage{titletoc}

\usepackage{amsthm}
\usepackage{thmtools, thm-restate}
\declaretheorem[numberwithin=section]{thm}
\declaretheorem[sibling=thm]{theorem}
\declaretheorem[sibling=thm]{lemma}
\declaretheorem[sibling=thm]{corollary}
\declaretheorem[numberwithin=section]{assumption}

\declaretheorem[]{challenged assumption}
\declaretheorem[]{definition}
\declaretheorem[]{proposition}
\declaretheorem[style=remark]{example}
\usepackage{xspace}

\definecolor{pierLink}{RGB}{189, 72, 80}
\definecolor{pierCite}{RGB}{0, 73, 112}
\definecolor{pierComment}{RGB}{100, 143, 255}

\crefname{as}{assumption}{assumptions}
\Crefname{as}{Assumption}{Assumptions}
\newcommand{\robset}{\mathcal{P}_{\text{rob}}}
\newcommand{\betaShat}{\hat{\beta}^{\cS}}

\newcommand{\PSM}{\hat{M}_{\mathrm{seen}}}
\newcommand{\PSMpop}{M_{\mathrm{seen}}}
\newcommand{\PSperpM}{\hat{R}\hat{R}^\top}
\newcommand{\PSperpMpop}{RR^\top}

\newcommand{\Cov}{\mathrm{Cov}}

\newcommand\cB{\mathcal{B}}

\newcommand\cM{\mathcal{M}}

\newcommand\cP{\mathcal{P}}

\newcommand\cN{\mathcal{N}}
\newcommand\cS{\mathcal{S}}

\def\mybigplus{\mathop{\mathchoice{
   \vcenter{\hbox to12bp{\vrule height15bp width0pt \pdfliteral{
   q 1 J .8 w 0 7.5 m 12 7.5 l S 6 1 m 6 14 l S Q
}\hss}}}{
   \vcenter{\hbox to12bp{\kern1bp\vrule height10bp width0pt \pdfliteral{
   q 1 J .65 w 0 5 m 10 5 l S 5 0 m 5 10 l S Q
}\hss}}}{+}{+}
}}

\newcommand{\Loss}{\mathcal{R}}
\newcommand{\R}{\mathbb{R}}

\newcommand\cov{\Sigma}

\newcommand{\Mtest}{M_{\mathrm{test}}}
\newcommand{\Mseen}{M_{\mathrm{seen}}}
\newcommand{\Mseenemp}{\hat{M}_{\mathrm{seen}}}

\newcommand{\Ecal}{\mathcal{E}}

\newcommand{\Ecaltrain}{\mathcal{E}_{\mathrm{train}}}

\newcommand{\Invset}{\Theta_{\mathrm{eq}}}

\newcommand\independent{\protect\mathpalette{\protect\independenT}{\perp}}
\def\independenT#1#2{\mathrel{\rlap{$#1#2$}\mkern2mu{#1#2}}}
\newcommand{\scalar}[2]{ \langle #1 , \; #2 \rangle}
\definecolor{tolblue}{HTML}{4477AA}
\newcommand{\fy}[1]{{\color{blue} Fanny: {#1}}}

\newcommand{\Lossrob}{\Loss_{\mathrm{rob}}}
\newcommand{\betarob}{\beta^{rob}}
\newcommand{\betarobarg}[1]{\beta^{rob}_{#1}}

\newcommand{\cBrobfull}{\cB^{rob}_{\Invset}}

\newcommand{\betabar}{\overline{\beta}}

\newcommand{\gammaprime}{\gamma'}

\newcommand{\wtilde}{\tilde{w}}

\newcommand{\betastar}{\beta^{\star}}

\newcommand{\noisecovxx}{\cov_{\eta}}
\newcommand{\noisecovxxstar}{\cov^\star_{\eta}}
\newcommand{\noisecovxxS}{\cov^{\cS}_{\eta}}
\newcommand{\noisecovxy}{\cov_{\eta,\xi}}
\newcommand{\noisecovxystar}{\cov^\star_{\eta,\xi}}
\newcommand{\noisecovxyS}{\cov^{\cS}_{\eta,\xi}}

\newcommand{\noisecovyy}{(\sigma_\xi)^2}
\newcommand{\noisecovyystar}{{(\sigma_\xi^\star)}^2}
\newcommand{\noisecovyyS}{({\sigma_\xi^\cS})^2}

\newcommand{\norm}[1]{\|#1\|}

\newcommand{\EE}{\mathbb{E}\,}

\newcommand{\prob}{\mathbb{P}}

\newcommand{\probtrain}{\cP_{\Ecaltrain}}
\newcommand{\ptrain}{\prob^{X,Y}_{\mathrm{train}}}
\newcommand{\probtrainarg}[1]{\cP_{#1, \Ecaltrain}}
\newcommand{\PeX}{\prob_e^{X}}
\newcommand{\PeXY}{\prob_e^{X,Y}}
\newcommand{\PrefXY}{\prob_{\mathrm{ref}}^{X,Y}}

\newcommand{\PeXYarg}[1]{\prob^{X,Y}_{#1, e}}
\newcommand{\PoXYarg}[1]{\prob^{X,Y}_{#1, 0}}
\newcommand{\PeA}{\prob^{A}_{e}}
\newcommand{\PeXYA}{\prob_e^{X,Y,A}}
\newcommand{\PXYgivenA}{\prob^{X,Y | A}}
\newcommand{\PeYgivenX}{\prob^{Y | X}_e}

\newcommand{\probtest}{{\prob_{\mathrm{test}}}}
\newcommand{\probtestXY}{\prob_{\mathrm{test}}^{X,Y}}
\newcommand{\probtestXYarg}[1]{\prob_{#1, \mathrm{test}}^{X,Y}}

\newcommand{\probtestA}{\prob^A_{\mathrm{test}}}

\newcommand{\thetastar}{\theta^{\star}}
\newcommand{\thetacS}{\theta^{\cS}}

\newcommand{\Id}{\text{Id}}

\newcommand{\Rtilde}{\tilde{R}}

\newcommand{\betaadv}{\betastar_{\mathrm{adv}}}
\newcommand{\thetaadv}{\theta_{\mathrm{adv}}}

\newcommand{\betarobpi}{\beta^{\mathrm{rob}}_{\Invset}}
\newcommand{\betarobpihat}{\hat{\beta}^{\mathrm{rob}}_{\Invset}}
\newcommand{\Lossrobpi}{\mathfrak{R}_{\mathrm{rob}}}

\newcommand{\Manchor}{M_{\mathrm{anchor}}}
\newcommand{\Mdrig}{M_{\mathrm{DRIG}}}
\newcommand{\Mnew}{M_{\mathrm{new}}}

\newcommand{\Cker}{C_{\mathrm{ker}}}
\newcommand{\Ckerhat}{\hat{C}_{\mathrm{ker}}
}
\newcommand{\range}{\mathrm{range}}

\newcommand{\Stot}{S_\mathrm{tot}}

\newcommand{\betaa}{\beta_{\text{anchor}}}

\newcommand{\Xe}{X_e}
\newcommand{\Xref}{X_{\mathrm{ref}}}

\newcommand{\Ye}{Y_e}
\newcommand{\Ae}{A_e}
\newcommand{\mue}{\mu_e}
\newcommand{\muref}{\mu_{\mathrm{ref}}}
\newcommand{\Sigmae}{\Sigma_e}
\newcommand{\Sigmaref}{\Sigma_{\mathrm{ref}}}
\newcommand{\Sigmastar}{\Sigma^{\star}}

\newcommand{\rnk}{\text{rank }}
\newcommand{\idset}{observationally equivalent set\,}

\newcommand{\betaOLS}{\beta_{\mathrm{OLS}}}
\newcommand{\betaOLShat}{\hat{\beta}_{\mathrm{OLS}}}
\newcommand{\betaahat}{\hat{\beta}_{\mathrm{anchor}}}
\newcommand{\betaDRIG}{\beta_{\mathrm{DRIG}}}

\newcommand{\Xtest}{X^{\text{test}}}

\newcommand{\Atest}{A_{\text{test}}}
\newcommand{\projM}{\Pi_{\cM}}
\newcommand{\mutest}{\mu_{\text{test}}}
\newcommand{\Sigmatest}{\Sigma_{\text{test}}}
\newcommand{\cSperp}{\cS^\perp}
\newcommand{\cShat}{\hat{\cS}}
\newcommand{\cSperphat}{\hat{\cS}^\perp}
\newcommand{\Shat}{\hat{S}}
\newcommand{\Rhat}{\hat{R}}

\newcommand{\RRhat}{\Rhat \Rhat^\top}
\newcommand{\PicShat}{\Pi_{\cShat}}
\newcommand{\PicSperphat}{\Pi_{\cSperphat}}
\newcommand{\betastarS}{\beta^{\cS}}
\newcommand{\betaS}{\beta^{\cS}}
\newcommand{\thetastarS}{\theta^{\cS}}
\newcommand{\betastarperp}{\beta^{\cSperp}}

\newcommand{\covAe}{\Sigma_e}

\newcommand{\gammath}{\gammaprime_{\mathrm{th}}}
\newcommand{\Robsetarg}[1]{\cP_{#1}(\Mtest)}

\newcommand{\supalpha}{\sup_{\substack{\alpha \in \cSperp, \\ \| \alpha \|_2 \leq \Cker} }}

\newcommand{\cmark}{\ding{51}}%
\newcommand{\xmark}{\ding{55}}%
\newcommand{\minimaxPIloss}{\mathfrak{M}(\Invset, \Mtest)}
\newcommand{\minimaxPIlossarg}[1]{\mathfrak{M}(\Invset, #1)}
\newcommand{\idRRs}{worst-case robust risk }
\newcommand{\idRR}{worst-case robust risk}
\newcommand{\IdRR}{Worst-case robust risk }

\newcommand{\idpred}{worst-case robust predictor}

\newcommand{\trace}{\mathrm{tr}\,}
\newcommand{\Iobs}{\mathcal{O}}
\newcommand{\Iint}{\mathcal{I}}
\newcommand{\Dtrain}{\mathcal{D}^{\text{train}}}

\newcommand{\cD}{\mathcal{D}}
\newcommand{\LL}{\mathcal{L}}
\newcommand{\LLref}{\mathcal{L}_{\mathrm{ref}}}
\newcommand{\LLinv}{\mathcal{L}_{\mathrm{inv}}}
\newcommand{\LLid}{\mathcal{L}_{\mathrm{id}}}

\newcommand{\Petaxi}{\mathbb{P}^{\eta,\xi}}

\definecolor{asparagus}{rgb}{0.53, 0.66, 0.42}
\definecolor{dollarbill}{rgb}{0.52, 0.73, 0.4}

\newcommand{\noteby}[3]{{\colorbox{#2}{\bfseries\sffamily\scriptsize\textcolor{white}{#1}}}{\textcolor{#2}{\sf\small\textit{#3}}}}
\newcommand{\julia}[1]{\noteby{Julia}{blue!80!green}{ #1}}

\renewcommand{\fy}[1]{\noteby{Fanny}{dollarbill}{ #1}}

\DeclareMathOperator*{\argmin}{arg\,min}

\renewcommand{\cmark}{\ding{51}}%
\renewcommand{\xmark}{\ding{55}}%

\renewcommand{\R}{\mathcal{R}}

\usepackage{tikz}
\usetikzlibrary{shapes, arrows.meta, positioning}

\usepackage{enumerate}
\usepackage{wrapfig}

\hypersetup{
    colorlinks,
   linkcolor={pierCite},
    citecolor={pierCite},
    urlcolor={pierCite}
}

\title{Achievable distributional robustness when the robust risk is only partially identified}

\author[1]{Julia Kostin}
\author[2]{Nicola Gnecco\thanks{This work was conducted while NG was at the Gatsby Computational Neuroscience Unit, University College London.}}
\author[1]{Fanny Yang}
\affil[1]{\small Department of Computer Science, ETH Zurich}
\affil[2]{\small Department of Mathematics, Imperial College London}

\begin{document}

\maketitle

\begin{abstract}

  In safety-critical applications, machine learning models should generalize well under worst-case distribution shifts, that is, have a small robust risk. Invariance-based algorithms 
  can provably take advantage of structural assumptions on the shifts when the training distributions are heterogeneous enough to identify the robust risk. However, in practice, such identifiability conditions are rarely
  satisfied -- a scenario so far underexplored in the theoretical literature. In this paper, we aim to fill the gap and propose to study the more general setting when the robust risk is only \emph{partially identifiable}. In particular, 
  we introduce the \emph{worst-case robust risk} as a new measure of robustness that is always well-defined regardless of identifiability.
  Its minimum corresponds to
  an algorithm-independent (population) minimax quantity that measures the \emph{best achievable robustness} under partial identifiability.
  While these concepts can be defined more broadly, 
  in this paper 
  we introduce and derive them explicitly for a linear model
  for concreteness of the presentation. 
  First, we show that existing robustness methods are provably suboptimal in the partially identifiable case.
  We then evaluate these methods and the minimizer of the (empirical) worst-case robust risk on real-world gene expression data and find a similar trend:
  the test error of existing robustness methods grows increasingly suboptimal as the 
  fraction of data from unseen environments increases, whereas accounting for partial identifiability allows for better generalization.
  \end{abstract}

  \section{Introduction}\label{sec:intro}

The success of machine learning methods typically relies on the assumption that the training and test data follow the same distribution. However, this assumption is often violated in practice. For instance, this can happen if the test data are collected at a different point in time or space, or using a different measuring device.
Without further assumptions on the test distribution, 
generalization under distribution shift is impossible.
However, practitioners often have partial information about the 
set of possible "shifts" 
that may occur during test time, inducing a set of \emph{feasible test distributions} that the model should generalize to. We refer to the resulting set as the \emph{robustness set}. 
When a probabilistic model for these possible test distributions is available or estimable, one may aim 
for good performance on a "typical" held-out distribution using a probabilistic framework. 
When no extra information is given, one possibility is to find a model $\beta$ that has a small risk $\Loss(\beta;\prob)$ on the \emph{hardest} feasible test distribution. More formally, we aim to achieve a small
robust risk defined by
\begin{equation}\label{eq:conventional-robust-risk}
   \Lossrob(\beta) \coloneqq \sup_{\prob \in \robset(\thetastar)} \Loss (\beta; \prob),
\end{equation}
where $\robset(\thetastar)$ corresponds to the robustness set which depends on 
some true parameter $\thetastar$.
In fact, this worst-case robustness aligns with security and safety-critical applications, where a small robust risk is necessary to confidently guard against possible malicious attacks.

\par

To find a robust prediction model that minimizes \eqref{eq:conventional-robust-risk}, existing lines of work in distributional robustness assume a known robustness set, i.e., full knowledge of the robust risk objective. They then focus on how to minimize the resulting objective. For instance, in distributionally robust optimization \citep{bental2013robust, duchi2021learning}, or relatedly, adversarial robustness  \citep{goodfellow2014explaining, madry2018towards}, the robustness set is chosen to be some neighborhood (w.r.t. to a distributional distance notion) of the training distribution $\prob$. When multiple training distributions are available, related works aim to achieve robustness against the set of all convex combinations of the training distributions \citep{mansour2008, meinshausen2014, sagawa2019distributionally}.  
In causality-oriented robustness  (see, e.g. \citep{buhlmann2020invariance,meinshausen2018causality,shen2023causalityoriented}) on the other hand, the robustness set is not explicitly given, but implicitly defined. Specifically, it is assumed that certain structural parameters (e.g. specific parts of the structural causal model) 
remain invariant across distributions, whereas other parameters shift. 
The robustness set may or may be not be fully known during training time, depending on the relationship of the varying parameters during training and shift time. 


If the robust risk objective is known, methods can be derived to estimate its minimizer. However, in many scenarios such procedures suffer from ineffectiveness.
For adversarial robustness for example, it is known that when the perturbations during training and test time differ, the robustness resulting from adversarial and standard training is comparable 
(see, e.g. \cite{Tramer19,kang19}). Similarly, invariance-based methods such as \citep{peters2016causal,rojas2018invariant,arjovsky2020invariant,krueger2021out} often exhibit no advantage over vanilla empirical risk minimization \citep{ahuja2020empirical,ahuja2020invariant}.
In both cases, one of the main failure reasons is that the robust objective, which the final model is evaluated on, is not known during training.  For instance, invariance-based methods often fail on new environments if the true invariant predictor is not identifiable (e.g. \citep{kamath2021does,rosenfeld2020risks}). As a possible solution, in some recent works \citep{rothenhausler2021anchor,shen2023causalityoriented}, the set of feasible test distributions is chosen in a specific way which renders the robust risk objective computable despite the non-identifiability of the invariant parameters. 
In total, the theoretical analysis in prior work remains rather of ``binary'' nature: it either analyzes the fully identifiable case in which invariance-based methods are successful, or simply discards non-identifiable scenarios as failure cases without further quantification of the limits of robustness.
In this paper, we aim to include the partially identifiable setting\footnote{Here, we mean partial identifiability of the robust risk, which is reminiscent of outputting uncertainty sets for a quantity of interest in the field of partial identification \cite{tamer2010partial, frake2023perfect}.} in our analysis and more specifically discuss the following question:

\begin{center}
\emph{
What is the optimal worst-case performance any model can have for given structural relationships between test and training distributions, and how do existing methods comparatively perform?}
\end{center}

When the robust risk is not identifiable from the collection of training distributions, 
we obtain a whole \emph{set} of possible objectives -- all compatible with the training distributions -- that includes 
the true robust risk. 
In this case, we are interested in the best achievable robustness for \emph{any algorithm}
that we capture in a quantity
called the \emph{\idRR}:
\begin{align}
\label{eq:identifiable-robust-risk}
    \Lossrobpi(\beta) := \sup_{\substack{\text{possible}\\ \text{true model $\thetastar$ }}}  \sup_{\prob \in \robset(\thetastar)} \Loss (\beta; \prob).
\end{align}
Note that $\Lossrobpi(\beta)$ is well-defined 
even when the standard robust risk is not identifiable -- 
it takes the supremum over the robust risks induced by 
all model parameters $\thetastar$
that are consistent with the given set of training distributions. 
Furthermore, the minimal value of the worst-case robust risk corresponds to the optimal worst-case performance in the partially identifiable setting.
Spiritually, this \emph{minimax population quantity} is reminiscent of the 
algorithm-independent limits in classical statistical learning theory \cite{yu1997assouad}.\footnote{In particular, extending \eqref{eq:identifiable-robust-risk} to its finite-sample counterpart would introduce a more natural extension of the classical minimax risk statistical learning theory.  In this work, we focus on identifiability aspects instead of statistical rates.}
Even though our partial identifiability framework 
can be evaluated for arbitrary 
modeling assumptions on the distribution shift (such as covariate/label shift, DRO, etc.), 
we present it
in \Cref{sec:setting} for a concrete linear setting for clarity of the exposition.
Specifically, the setting is motivated by structural causal models (SCMs) with unobserved confounding (cf. \cref{sec:setting}), similar to the setting of instrumental variables (IV) and anchor regression \citep{rothenhausler2021anchor,saengkyongam2022exploiting}. Concurrent to our work, \cite{bellot2022partial} proposed a similar framework derived explicitly from structural causal models, with quantities that are closely related to our worst-case robust risk and its minimizer. 

In \Cref{sec:main-results}, we derive the \idRR~\eqref{eq:identifiable-robust-risk} and its minimum for the linear setting, and show theoretically and empirically that the ranking and optimality of different robustness methods change drastically in identifiable vs. partially identifiable settings.  Further, although the worst-case robust predictor derived in the paper is only provably optimal for the linear setting, experiments on real-world data in \Cref{sec:experiments} suggest that our estimator may significantly improve upon other invariance-based methods in more realistic scenarios. 
Our experimental results provide evidence that evaluation and benchmarking on partially identifiable settings are important for determining the effectiveness of robustness methods.

  \section{Setting}
  \label{sec:setting}
In this section, we state the concrete distributional setting on which we introduce our partial identifiability framework. In particular, we consider a data generating process, motivated by  structural causal models (SCMs), that allows for hidden confounding, i.e., spurious correlations between the covariates $X$ and the target $Y$. We describe the structure of the distribution shifts occurring in the training and test environments, which is reminiscent of  interventions in causal models.
Finally, we introduce our framework for distributional robustness that allows for partial identifiability and define the \emph{\idRR} – for any given model, it corresponds to the maximum robust risk among all possible robust risks compatible with the training distributions. 
\subsection{Data distribution and a model of additive  environmental shifts}\label{sec:training-data}

\textbf{Data generating process (DGP).} We first describe the data-generating mechanism that underlies the distributions of all environments $e \in \Ecal$ that may occur during train or test time.
For each environment $e \in \Ecal$, we observe the random vector  $(\Xe, \Ye) \sim \PeXY$ consisting of input covariates $\Xe \in \R^d$ and  a target variable $\Ye \in \R$ which satisfy the following data generating process: 
\begin{align}
\label{eqn:SCM}
\begin{split}
    \Xe &= \Ae + \eta, \\ 
    \Ye &= {\betastar}^\top \Xe + \xi,
\end{split}
\end{align}
where $\Ae \in \R^d$, $(\eta,\xi) \in \R^{d+1}$ are random vectors such that $\Ae \sim \PeA$, $(\eta,\xi) \sim \Petaxi$ with finite first and second moments, and $\Ae \independent (\eta, \xi)$ for all $e \in \Ecal$. 

Note how in this setting, the parameter $\betastar$ and the distribution $\Petaxi$ of the noise vector $(\eta,\xi)$ remain invariant across environments. Without loss of generality, we assume that the noise $(\eta,\xi)$ has mean zero. Our linear setting
is, in general, more challenging than the standard linear regression setting where $\eta \independent \xi$: due to possible dependencies between $\eta$ and $\xi$ (induced by, e.g., \emph{hidden confounding/spurious features}), classical estimators, such as the ordinary least squares, are biased away from the true parameter~$\betastar$. 
Denote by $\Sigmastar \coloneqq \Cov((\eta,\xi))$ the joint covariance of the noise vector $(\eta,\xi)$, which can be written in block form as  $\Sigmastar = \begin{pmatrix}\noisecovxxstar& \noisecovxystar\\{\noisecovxystar}^\top&\noisecovyystar\end{pmatrix}$ and which we assume to be full-rank. In the following, we denote the concatenation of these two invariant parameters by $\thetastar := (\betastar, \Sigmastar) \in \Theta \subset  \R^{d}\times \R^{(d+1)\times(d+1)}$, where $\Theta$ is an appropriate parameter space. 

\textbf{Structure of the distribution shifts.}
Note that in the DGP, the distribution shifts across $\PeXY$ are induced solely by changes in the distribution of the variable $\Ae$, whose mean and covariance matrix we denote by $\mu_e \coloneqq \EE[\Ae]$ and $\covAe \coloneqq\Cov [\Ae]$, respectively. 
In general, we allow for degenerate and deterministic shifts, i.e. the covariance $\covAe$ can be singular or zero.  We remark that although the additive shift structure in \cref{eqn:SCM} allows the joint distribution $\PeXYA= \PeA \times \PXYgivenA$ to change solely via $\PeA$, our distribution shift setting is more general than covariate shift: due to the noise variables $\eta$ and $\xi$ being potentially dependent, both the marginal $\PeX$ and the conditional distribution $\PeYgivenX$ can change across environments. 

\textbf{Training and test-time environments.} 
Throughout the paper, we assume that we are given the collection of training distributions $\probtrainarg{\thetastar} = \{ \PeXYarg{\thetastar} \}_{e \in \Ecaltrain}$, where $\Ecaltrain$ denotes the index set of training environments. We omit $\thetastar$ when it is clear from the context. Further, for ease of exposition, we  assume that $\Ecaltrain$ contains a reference (unshifted) environment $e = 0$ with $A_0 = 0$ a.s. In \cref{sec:apx-extension}, we discuss how our results apply if this condition is not met.
During test time, we expect to observe a new, previously unseen distribution $\probtestXY$ which is induced by the DGP \eqref{eqn:SCM} and a shift random variable $\Atest \sim \probtestA$, with corresponding finite mean $\mutest$  and covariance $\Sigmatest$.

Even though we do not have access to $\probtestXY$ during training,
the practitioner might have some information about the possible
shift distributions $\probtestA$ that may occur during test time. 
As an example, we may 
only have information about the maximum possible magnitude and approximate direction of the test-time mean shift $\EE[\Atest]$. 
In this work, we assume that we are given 
an upper bound on the second moment of the shift variable, represented by a positive semidefinite (PSD) matrix $\Mtest\succeq 0$ such that 
\begin{align}\label{eqn:testAbound}
    \EE[ \Atest {\Atest}^\top] = \Sigmatest + \mutest {\mutest}^\top \preceq \Mtest.
\end{align}
In practice, there may be different degrees of knowledge of the feasible set of shifts -- when no knowledge is available, one can always choose the most "conservative" bound $\Mtest$ with the range equal to $\R^{d}$ and large eigenvalues. The more information is available, the smaller the feasible set of test distributions becomes. For instance, when 
the test distribution 
$\prob_{\text{test}}^X$
of $X$ is available during training (as in the \emph{domain adaptation} setting \cite{shimodaira2000improving}), one can directly compute the optimal shift upper bound via $\Mtest =  \EE[\Xtest {\Xtest}^\top] $.
In existing literature, $\Mtest$ is often proportional to the pooled first or second moment of the training shifts, for instance $\Mtest = \gamma \sum_{e \in \Ecaltrain} w_e \mu_e \mu_e^\top$ in discrete anchor regression \cite{rothenhausler2021anchor} or $\Mtest = \gamma \sum_{e \in \Ecaltrain} w_e (\mu_e \mu_e^\top + \Sigma_e)$ in causality-oriented robustness with invariant gradients \cite{shen2023causalityoriented}.
Here, $w_e$ are the weights representing the probability of a datapoint being sampled from the environment $e$. As will become apparent in the next sections, our population-level results are not impacted by the weights $w_e$, which we thus omit in the following.

We now provide an example based on structural causal models (SCM) that falls under the aforementioned distrubtion shift setting. 
\begin{example}
    Consider the structural causal model \eqref{eqn:SCM-new} and its induced graph in \cref{fig:ex-scm}. In this model, interventions on the variable $Z$ correspond to soft interventions on the covariates $X$. Additionally, the exogenous noise vector $(\varepsilon_X, \varepsilon_Y, \varepsilon_H)$ and the intervention variable $Z$ are mutually independent. This model is the basis of multiple causality-oriented robustness works, e.g. \citep{rothenhausler2021anchor,shen2023causalityoriented}.
    Let $\betastar \coloneqq B_{YX}^\top$ and $\xi \coloneqq B_{YH}H + \varepsilon_Y$. Then, from~\eqref{eqn:SCM-new}, we obtain $Y = B_{YX} X + (B_{YH} H + \varepsilon_Y)=X^\top\betastar + \xi$.
    Suppose that $\mathbf{I} - \mathbf{B}$ is invertible and let $\mathbf{C}\coloneqq (\mathbf{I} - \mathbf{B})^{-1}$with entries $C_{XX}, C_{XY}$, etc. Define $A \coloneqq C_{XX} Z$ and $\eta \coloneqq C_{XX} \varepsilon_X + C_{XY} \varepsilon_Y + C_{XH} \varepsilon_H$. Then, we can write $X = A + \eta$.
    Since shifts in the distribution of $Z$ induce shifts in the distribution of $A$, a collection of interventions $\{Z_e\}_{e \in \Ecaltrain}$ translates into a collection of additive shifts $\{A_e\}_{e\in\Ecaltrain}$ and gives rise to training distributions varying with the environment $e$. In summary, our DGP \eqref{eqn:SCM} includes the classical setting of causality-oriented robustness as depicted in \cref{fig:ex-scm}. 
    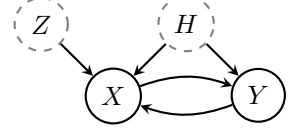
\begin{figure}[h!]\label{fig:fig1setting}
\begin{subfigure}{.55\textwidth}
\begin{align}
\label{eqn:SCM-new}
\begin{split}
    \begin{pmatrix}
        X\\
        Y\\
        H
    \end{pmatrix}
    =
    \underbrace{\begin{pmatrix}
        B_{XX} & B_{XY} & B_{XH}\\
        B_{YX} & 0 & B_{YH}\\
        0 & 0 & 0
    \end{pmatrix}
    }_{
    \mathbf{B}
    }
    \begin{pmatrix}
        X\\
        Y\\
        H
    \end{pmatrix}
    +
    \begin{pmatrix}
        Z + \varepsilon_X\\
        \varepsilon_Y\\
        \varepsilon_H    
    \end{pmatrix}
\end{split}
\end{align}
\end{subfigure}
\begin{subfigure}{.45\textwidth}
  \centering
  \begin{tikzpicture}
  \tikzset{
    observed/.style={circle, draw=black, fill=white, thick, minimum size=7mm},
    latent/.style={circle, draw=gray, fill=white, dashed, thick, minimum size=7mm},
    arrow/.style={->, thick, >=stealth},
    bidirectional/.style={<->, thick, black}
  }

  \node[observed] (X)                        {$X$};
  \node[latent, above left=.6cm of X] (A) {$Z$};
  \node[latent, above right=.6cm of X] (H) {$H$};
  \node[observed, below right=.6cm of H] (Y) {$Y$};  

  \draw[arrow] (A) -- (X);
  \draw[arrow] (X) to[bend left=20] (Y);
  \draw[arrow] (Y) to[bend left=20] (X);
  \draw[arrow] (H) -- (X);
  \draw[arrow] (H) -- (Y);
\end{tikzpicture}
\end{subfigure}%
\caption{\small{(Left) SCM with hidden confounding and (right) induced graph. The model allows for an arbitrary causal structure of the observed variables $(X,Y)$, as long as $\mathbf{I} - \mathbf{B}$ is invertible, e.g. when the underlying graph is acyclic. The shifts across different distributions are captured via shift interventions on $X$. However, the model does not allow for interventions on the target variable $Y$ or hidden confounders $H$.}}
\label{fig:ex-scm}
\end{figure}
\end{example}

\subsection{The robust risk}\label{sec:formulation-distributional-robustness}
Our goal is to find a model $\beta \in \R$ using the training distributions
that has a small risk
over the robustness set.
In this paper, we consider as risk function the expected square loss  $\Loss(\beta; \prob) \coloneqq \EE_{\prob} [(Y - \beta^\top X)^2]$.
In our setting, given a test shift upper bound $\Mtest$ defined in \Cref{eqn:testAbound}, the robustness set corresponds to
\begin{align}\label{eqn:robustness-set}
    \Robsetarg{\thetastar} := \{ \probtestXYarg{\thetastar}: \:  \EE [ \Atest {\Atest}^\top ] \preceq \Mtest \},
\end{align}
yielding the corresponding robust risk that reads

\begin{align}\label{eqn:robust-risk}
    \Lossrob(\beta;\thetastar, \Mtest) \coloneqq \sup_{\prob \in \Robsetarg{\thetastar}} \Loss( \beta; \prob). 
\end{align}

We call the minimizer of the robust risk 
 $\betarob_{\thetastar} \coloneqq \argmin_{\beta \in \R^d} \Lossrob(\beta; \thetastar,\Mtest) $ the \emph{robust predictor}.
%
 For the squared loss and linear model, the robust risk and the robust predictor can be computed in closed form and \emph{solely} depend on $\Mtest$ and the invariant parameters $\thetastar = (\betastar,\Sigmastar)$, 
and not on other properties of the distributions:
\begin{align}\label{eqn:formula-robust-predictor}
\begin{split}
        \Lossrob(\beta, \thetastar, \Mtest) &=  (\betastar - \beta)^\top (\noisecovxxstar + \Mtest)(\betastar - \beta) + 2(\betastar-\beta)^\top\noisecovxystar + {\noisecovyystar}. 
        \\
        \betarobarg{\thetastar}  &= \argmin_{\beta \in \R^d} \Lossrob(\beta, \thetastar, \Mtest) = \betastar + (\Mtest + \noisecovxxstar)^{-1} \noisecovxystar.
\end{split}
\end{align}
\textbf{Induced equivalence of data generating processes.} This observation motivates us to define an equivalence relation between two data-generating processes that holds whenever they induce the same robust risk for any model $\beta$ and shift upper bound $\Mtest$. Specifically, observe that $\mathrm{DGP}_1$ and $\mathrm{DGP}_2$ induce the same robust risks for all $\Mtest$ and $\beta$ iff $\betastar_1 = \betastar_2$ and $\Petaxi_1 \cong \Petaxi_2$, where $\cong$ denotes the equivalence of distributions based on equality of their first and second moments. 
Thus, in the following, we treat our data-generating process as uniquely defined by $\thetastar =(\betastar, \Sigmastar)$ up to this equivalence relation.

In practice, the model parameters $\thetastar$ typically cannot be identified from the training distributions,
and the robust risk $\Lossrob$ can only be computed for specific combinations of training and test shifts, studied, e.g., in \citep{rothenhausler2021anchor,shen2023causalityoriented}. In the next section, we describe concepts that allow us to 
reason about robustness even when the robust risk is only partially identifiable.

\subsection{Partially identifiable robustness framework}\label{sec:prop-invariant-set}
We start by formally introducing the general notion of partial identifiability for the robust risk.
The following notion of \emph{observational equivalence} of parameters is reminiscent of the
corresponding notion in the econometrics literature \cite{Dufour2010Econometrics}:
\begin{definition}[Observational equivalence]
 Two model parameter vectors $\theta_1 = (\beta_1, \Sigma_1)$ and $\theta_2 = (\beta_2, \Sigma_2)$ are \textbf{observationally equivalent} with respect to a set of shift distributions $\{ \PeA: e \in \Ecaltrain \}$\footnote{The distributions $\PeA$ are to be understood up to the equivalence relation $\cong$. In general, the distributions $\PeA$ are unknown, since the shift variables $\
 \Ae$ are unobserved. However, in our setting, $\PeA$ can be identified up to the second moment because of the reference environment.}  if they induce the same set $\probtrainarg{\theta}$ of training distributions over the observed variables $(X_e,Y_e)$ as 
described in \Cref{sec:training-data},
 i.e.
 \begin{align*}
    \PeXYarg{\theta_1} \cong \PeXYarg{\theta_2} \: \text{for all } e \in \Ecaltrain.
 \end{align*}
    By observing $\probtrainarg{\thetastar}$, we can identify the model parameters $\thetastar$ up to the \textbf{\idset} defined as 
    \begin{align*}
       \Invset := \{ \theta = (\beta, \Sigma) \in \Theta: \probtrainarg{\theta} \cong \probtrainarg{\thetastar} \}. 
    \end{align*}
\end{definition}

In general, the \idset is not a singleton, 
that is, $\thetastar$ is not identifiable from the collection of training environments $\probtrainarg{\thetastar}$.
However, prior work has exclusively considered test shifts $\Mtest$ that still allow identifiability of the robust risk nonetheless, depicted in \cref{fig:sub1-identified} and discussed again in \cref{sec:comp-with-finite-robustness-methods}. 
In this work, we argue for analyzing the more general partially identifiable setting, where set-identifiability of the invariant parameter $\thetastar$
only allows us to compute a superset
of the robustness set
\begin{align*}
    \Robsetarg{\Invset} := \bigcup_{\theta \in \Invset} \Robsetarg{\theta} \supset\Robsetarg{\thetastar} 
\end{align*}
and correspondingly, a set of robust risks $\{\Lossrob(\beta; \theta,\Mtest): \, \theta \in \Invset \}$.
In this case, we would still like to achieve the 
``best-possible'' robustness, that is 
the test shift robustness
for the ``hardest-possible'' 
parameters that could have induced the observed training distributions.
\begin{definition}[\IdRR and the minimax quantity]\label{def:PI-robust-loss}
    For the data model in \Cref{eqn:SCM}, the \idRRs is defined as
    \begin{align}
    \label{eqn:PI-robust-loss} 
        \Lossrobpi(\beta; \Invset, \Mtest) :=   \sup_{\theta \in \Invset} \Lossrob  (\beta;\theta, \Mtest). 
    \end{align}
The optimal robustness on test shifts bounded by $\Mtest$ given training data $\probtrainarg{\thetastar}$ is described by the minimax quantity 
\begin{align}
\label{eqn:minimax-quantity}
    \minimaxPIloss  = \inf_{\beta \in \R^d}\Lossrobpi( \beta; \Invset, \Mtest ).
\end{align}
When a minimizer of \Cref{eqn:minimax-quantity} exists, we call it the \idpred\, defined by
    \begin{equation}
        \betarobpi = \argmin_\beta \Lossrobpi(\beta; \Invset, \Mtest)
    \end{equation}
\end{definition}
\begin{figure}
\centering
\begin{subfigure}{.5\textwidth}
  \centering
  \includegraphics[width=\textwidth]{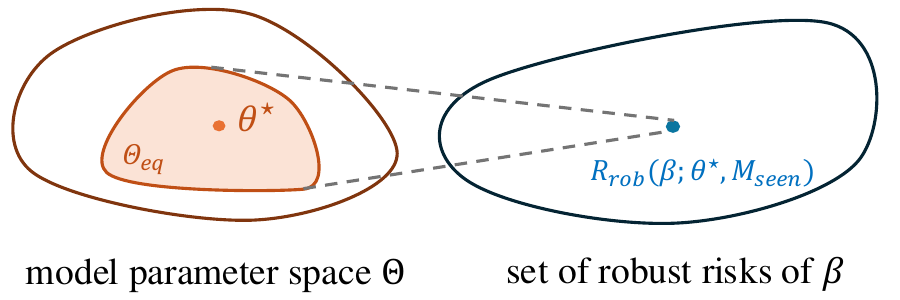}
  \caption{Identifiable robust risk}
  \label{fig:sub1-identified}
\end{subfigure}%
\begin{subfigure}{.5\textwidth}
  \centering
  \includegraphics[width=\textwidth]{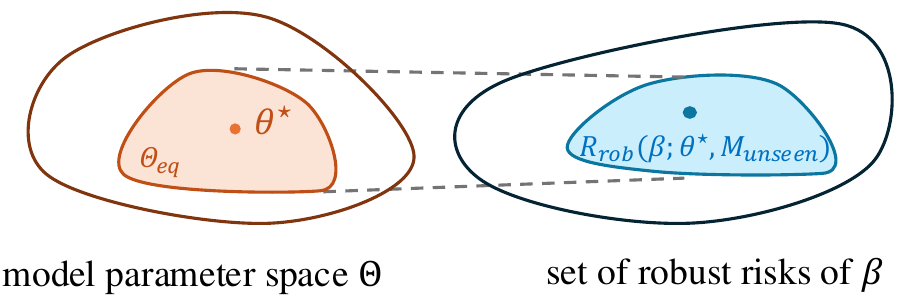}
  \caption{Partially identifiable robust risk}
  \label{fig:sub2-nonidentified}
\end{subfigure}
\caption{\small{Relationship between identifiability of the model parameters and identifiability of the robust risk. (a) The classical scenario where the test shift upper bound $\Mtest = M_{\text{seen}}$ is contained in the span of training shifts so that 
the robust risk is point-identified. (b) The more general scenario of this paper, where $\Mtest=M_{\text{unseen}}$ contains new shift directions and where only a set can be identified in which the true robust risk lies. }}
\label{fig:test}
\end{figure}

In the next sections, we 
explicitly compute these quantities for the linear setting of \cref{sec:training-data}.
This will allow us to compare
the best achievable robustness in the partially identified case with the guarantees of prior methods in this setting.
  \section{Theoretical results for 
  the linear setting}
  \label{sec:main-results}
  

We now compute the \idRRs \eqref{eqn:PI-robust-loss} and derive a lower bound for  the minimax quantity \eqref{eqn:minimax-quantity} in the linear additive shift setting of \cref{sec:training-data}.
We then compare the \idRRs of some existing robustness methods and ordinary least squares (OLS) with the minimizer of the \idRRs both theoretically and empirically.

\subsection{Minimax robustness results for the linear setting}\label{sec:main-results-minimax}
\label{sec:identifiability-linear-SCM}

The degree to which the model parameters $\thetastar$  in the linear additive shift setting \eqref{eqn:SCM} can be identified depends on the number of environments and the total rank of the moments of the additive shifts. For structural causal models, this is well-studied, for instance, in the instrumental variable (IV) regression literature \citep{ amemiya1985advanced, bowden1990instrumental}. 
As we show in \Cref{prop:invariant-set}, the true parameter $\betastar$ can \emph{only} be identified along the directions of the training-time mean and variance shifts $\mue$ and $\Sigmae$. 
Therefore, if not enough shift directions are observed, $\betastar$ is merely \emph{set-identifiable}, leading to 
 set-identifiability of the robust prediction model \eqref{eqn:formula-robust-predictor}. 
More formally, we denote by $\cS$ the subspace consisting of all \emph{additive shift directions seen during training}:
\begin{align}\label{eqn:def-S}
    \cS := \range \left[ \sum_{e \in \Ecaltrain} \left( \Sigma_e + \mu_e \mu_e^\top\right) \right],
\end{align} 
and by $\cSperp$ its orthogonal complement.
The definition of the space $\cS$ induces an orthogonal decomposition of the true parameter $\betastar = \betastarS + \betastarperp$. The \emph{identifiable part} $\betastarS$ then uniquely defines a set of \emph{identified model parameters} that reads 
\begin{align*}
    \thetastarS := (\betaS, \noisecovxxS, \noisecovxyS, \noisecovyyS) = (\betastarS, \noisecovxxstar, \noisecovxystar + \noisecovxxstar \betastarS, \noisecovyystar + 2 \langle \noisecovxystar, \betastarS\rangle + \langle \betastarS, \noisecovxxstar \betastarS\rangle)
\end{align*}
that can be computed from the training distributions. For the following results, we assume a similar decomposition of the test shift upper bound $\Mtest$ which is essentially a decomposition into "seen" and "unseen" directions.
\begin{assumption}[Structure of $\Mtest$]\label{as:Mtest-structure}
   We assume that $\Mtest = \gamma \Mseen + \gammaprime R R^\top$, where $\gamma, \gammaprime \geq 0$, $\Mseen$ is a PSD matrix satisfying $\range(\Mseen) \subset \cS$ and $R$ is a semi-orthogonal matrix satisfying $\range(R) \subset \cSperp$.
\end{assumption}
In the next proposition, we show that the model parameters and robust predictor can be identified up to a neighborhood around $\thetastarS$. 

 

\begin{proposition}[Identifiability of model parameters and robust predictor]
\label{prop:invariant-set} Suppose that the set of training and test distributions is generated according to \cref{sec:training-data} and Assumption~\ref{as:Mtest-structure} holds.
Then,
\begin{enumerate}[(a)]
 \item 
 the model parameters 
 generating the training distribution \eqref{eqn:SCM} can be identified up to the following \idset: 
\begin{align}\label{eqn:def-invariant-set}
  \Invset =  \Theta \cap \{ \betastarS + \alpha, \noisecovxxstar, \noisecovxyS - \noisecovxxstar \alpha, \noisecovyyS - 2 \alpha^\top \noisecovxyS + \alpha^\top \noisecovxx \alpha \colon \alpha \in \cSperp \}  \ni \thetastar;
\end{align}
\item 
the robust predictor $\betarob_\theta$ as defined in \cref{eqn:formula-robust-predictor} is identified up to the set
    \begin{align}\label{eqn:def-rob-pred-identif}
    \cBrobfull \cap \{ \betastarS + (\Mtest + \noisecovxxstar)^{-1} \noisecovxyS +  (\Mtest + \noisecovxxstar)^{-1} \alpha\colon \, \alpha \in \range(R)  \}   \ni \betarob_{\theta}, 
    \end{align}
\end{enumerate}
where $\cBrobfull = \{ \betarob_{\theta}: \theta \in \Invset \}$. 
\end{proposition}
The proof of \cref{prop:invariant-set} is provided in \cref{sec:apx-proof-invariant-set}. 
\cref{prop:invariant-set} covers two well-known settings: If we observe a rich enough set $\probtrain$ of training environments such that $\cS = \R^d$, the model parameters are uniquely identified, corresponding to the setting of full-rank instruments \cite{amemiya1985advanced}. From a dual perspective, for a given set of training environments, \emph{the robust predictor is  identifiable whenever the test shifts are in the same direction as the training shifts}, i.e. $\range (\Mtest) \subset \cS$ and $R=0$ -- this holds even  when the invariant parameters are not identifiable and $\cS \neq \R^d$.
This is the setting considered e.g. in anchor regression \cite{rothenhausler2021anchor} and discussed again in~\Cref{sec:comp-with-finite-robustness-methods} and \cref{sec:apx-anchor-connections}.
So far, we have described how the identifiability of the robust prediction model depends on the structure of both the training environments (via the space $\cS$) and the test environments (via $\Mtest$). 
We now aim to compute the smallest achievable robust loss for the general partially identifiable setting, which allows for $R \neq 0$.
In particular, we provide a lower bound on the \emph{best-possible achievable distributional robustness} formalized by the minimax quantity \eqref{eqn:minimax-quantity}.
First observe that without further assumptions on the parameter space $\Theta$, the \idset is unbounded, and the \idRRs \eqref{eqn:PI-robust-loss} can be infinite.
The following boundedness assumption allows us to provide a fine-grained analysis of robustness in a partially identified setting.
\begin{assumption}[Boundedness of the causal parameter]\label{as:bounded-betastar}
    There exists a constant $C > 0$ such that any parameter $\beta$ in the DGP  \eqref{eqn:SCM} is norm-bounded by $C$, i.e.  $\norm{\beta}_2 \leq C$. 
\end{assumption}
Furthermore, we denote by $\Cker \coloneqq \sqrt{C^2 - \| \betastarS\|^2}$ the maximum norm of the non-identified part of the true parameter $\betastar$. Finally, recall that the reference distribution $\PoXYarg{\thetastar}$ is observed and hence identifiable. 

\begin{theorem}\label{thm:pi-loss-lower-bound}
    Assume that the training and test data follow the data-generating process as in \cref{sec:training-data} 
    and $\Mtest$ satisfies Assumption~\ref{as:Mtest-structure} for some $\Mseen,R$ with $\range (\Mseen) \subset \cS$, $\range (R) \subset \cSperp$. 
    Further, let Assumption~\ref{as:bounded-betastar} hold with parameter $C$.
    The \idRRs \eqref{eqn:PI-robust-loss}
    is then given by 
    \begin{equation}\label{eqn:ID-robust-risk-formula}
        \Lossrobpi(\beta;\Invset,\Mtest) = \gammaprime  (\Cker + \| R^\top \beta \|_2)^2 + \gamma (\betastarS - \beta)^\top \Mseen (\betastarS - \beta) + \Loss(\beta;\PoXYarg{\thetastar}).
    \end{equation}
The minimax quantity in \cref{eqn:minimax-quantity} is lower bounded as follows:
\begin{align} 
    &\minimaxPIloss\begin{cases}
            = \gammaprime \Cker^2 + \min_{R^\top \beta = 0} \Lossrob(\beta;\thetacS, \gamma \Mseen), & \text{if} \: \gammaprime \geq \gammath;\\
            \geq \gammaprime  \Cker^2 + \min_{\beta \in \R^d} \Lossrob(\beta;\thetacS,\gamma \Mseen) , & \text{else},  
        \end{cases}\label{eqn:thm-minimax} 
\end{align}
where $\gammath = \frac{(\kappa(\noisecovxxstar) + 1) \norm{R R^\top \noisecovxyS}}{\Cker}$\footnote{$\kappa$ denotes the condition number of the covariance matrix $\noisecovxxstar$.}. 
Moreover, for small unseen shifts 
\begin{align}\label{eqn:minimax-limit}
    \lim_{\gammaprime \to 0} \frac{\minimaxPIloss}{\gammaprime} = (\Cker + \| R R^\top {\noisecovxxstar}^{-1} \noisecovxyS\| )^2.
\end{align}

\end{theorem}
We prove \cref{thm:pi-loss-lower-bound} in \cref{sec:apx-proof-of-main-prop}. First, in the case of no new test shifts where $\gammaprime = 0$ (as it appears in prior work \citep{rothenhausler2021anchor,shen2023causalityoriented}) we can plug in the robust risk \Cref{eqn:robust-risk} into \Cref{eqn:thm-minimax} to observe the following: as the strength $\gamma$ of the shift grows, the optimal robust risk saturates.
On the other hand, 
if $\gammaprime \neq 0$, i.e., the test shift contains new directions w.r.t. to the training data, the best achievable robustness $\minimaxPIloss$ grows linearly with $\gammaprime$.
Further note that for $\gammaprime \geq \gammath$, we have a  tight expression for the minimax quantity
and the worst-case robust predictor $\betarobpi$ can be explicitly computed 
(cf. \Cref{sec:apx-proof-of-main-prop}) and is \emph{orthogonal} to the space $\range (R)$ of new test shift directions. 
In other words, for large shifts in new directions, the optimal robust model would "abstain" from prediction in those directions. For smaller shifts $\gammaprime$, $\betarobpi$ gradually utilizes more information in the non-identified directions, thus interpolating between maximum predictive power (OLS) and robustness w.r.t. new directions (abstaining). 
The model $\betarobpi$ is a population quantity that is identifiable from the collection of training \emph{distributions}. When only finite samples are available, we discuss in \cref{sec:apx-empirical-estimation}
how we can compute the worst-case robust estimator by minimizing
an empirical loss function
that can be computed from multi-environment data. 

\subsection{Theoretical analysis of existing finite robustness methods}\label{sec:comp-with-finite-robustness-methods}
We now evaluate existing finite robustness methods in our partial identifiability framework and 
characterize their (sub)optimality in different scenarios.
A spiritually similar systematic comparison of domain adaptation methods is presented in \cite{chen2021domain}, however, in our setting, the robust risk is not identifiable from data. 
Concretely, we compare discrete anchor regression \cite{rothenhausler2021anchor} and pooled OLS estimators \footnote{In \cref{sec:apx-anchor-connections}  we show that analogous results hold for continuous anchor regression and the method of distributionally robust invariant gradients (DRIG) \cite{shen2023causalityoriented}.} with the minimax quantity in \Cref{thm:pi-loss-lower-bound}. We consider the same scenario as in discrete anchor regression, which is a the specific case of the setting in \Cref{eqn:SCM}, where for each environment $e$, $\Ae$ is just a mean shift with variance $0$.
In addition, discrete anchor regression assumes that the environment variable $E \in \Ecaltrain$ follows a probability distribution with $\prob[E = e] = w_e$. 
The discrete anchor setting then corresponds to setting a test shift upper bound $\Mtest = \gamma\Manchor$ for some $\gamma>0$ (cf. \Cref{eqn:testAbound}) with  $\Manchor = \sum_{e \in \Ecaltrain} w_e \mue \mue^\top$. 
The (oracle) discrete anchor regression estimator minimizes the robust risk and reads 
\begin{align}
\label{eq:betaanchor}
    \betaa = \argmin_{\beta \in \R^d} \Lossrob(\beta; \thetastar, \gamma \Manchor).
\end{align}
The pooled ordinary least squares (OLS) estimator $\betaOLS$ corresponds to $\betaa$ with $\gamma = 1$.  We observe that the test shifts bounded by $\gamma \Manchor$ are fully contained in the space of identified directions $\cS$, since $\cS = \range (\cup_{e \in \Ecaltrain} \mue \mue^\top) = \range (\Manchor)$.  Thus, according to \cref{prop:invariant-set}, the robust risk and robust predictor $\betaa$ are identifiable for all $\gamma > 0$. In the next corollary, we compute worst-case robust risk of both $\betaa$ and $\betaOLS$
with respect to the more general shifts bounded by $\Mtest := \gamma \Manchor + \gammaprime R R^\top$, thus possibly including unseen shifts consisting of additional unseen shifts $\range (R) \subset \cSperp$.
\begin{corollary}[Worst-case robust risk of the anchor regression estimator]\label{cor:estimators}
Assume that the test shift upper bound is given by $\Mtest := \gamma \Manchor + \gammaprime R R^\top$. Let 
$\prob^{X,Y}_{\mathrm{train}} = \sum_{e\in\Ecaltrain} w_e \PeXY$ be the pooled training distribution. 
Then the general worst-case robust risk is given by 
\begin{align*}
    \Lossrobpi(\beta; \Invset, \Mtest) = \gammaprime (\Cker + \| R^\top \beta \|_2)^2 + (\gamma-1)(\betastarS-\beta)^\top \Manchor (\beta - \betastarS) + \Loss(\beta,\ptrain).
\end{align*}
Furthermore, for the anchor and OLS predictor, respectively, it holds that there exist functions  $c_1(\gamma),c_2(\gamma),c_3(\gamma)$ independent of $\gamma'$ such that 
\begin{equation*}
    \begin{aligned}
        \Lossrobpi(\betaa; \Invset, \Mtest) &= (\Cker + \| R R^\top (\noisecovxxstar + \gamma \Manchor)^{-1} \noisecovxyS \| )^2\gammaprime + c_1(\gamma); \\ 
        \Lossrobpi(\betaOLS; \Invset,\Mtest) &= (\Cker + \| R R^\top (\noisecovxxstar + \Manchor)^{-1} \noisecovxyS \| )^2\gammaprime + c_2(\gamma).
    \end{aligned}
\end{equation*}
In contrast, the best achievable robustness reads
\begin{equation*}
\begin{aligned}
    \minimaxPIlossarg{\Mtest} &= \Cker^2 \gammaprime + c_3(\gamma), \: \text{ if } \gammaprime \geq \gammath; \\  \lim_{\gammaprime \to 0} \minimaxPIlossarg{\Mtest} / \gammaprime &= (\Cker + \| R R^\top (\noisecovxxstar + \gamma \Manchor)^{-1} \noisecovxyS\| )^2.
    \end{aligned}
\end{equation*}
\end{corollary}
Observe that the \idRRs in the extended anchor regression setting is equal to the anchor regression risk with an additional non-identifiability penalty term $\gammaprime (\Cker + \| R^\top \beta \|_2)^2$.
The anchor regression estimator is optimal in the limit of vanishing unseen shifts but, for any $\gamma$, significantly deviates\footnote{Notice that the term $ \| R R^\top (\noisecovxxstar + \gamma \Manchor)^{-1} \noisecovxyS\|$ is only zero (yielding the minimax risk) if $\Manchor$ is full-rank (implying $R = 0$), or if $(\noisecovxxstar + \gamma \Manchor)^{-1} \noisecovxyS \perp R$, implying that there is no meaningful confounding in the unseen test shift directions. Otherwise, it can be strictly bounded from below as $\noisecovxxstar$ is full-rank.} from the best achievable robustness for larger unseen shifts $\gammaprime \geq \gammath$. 
Moreover, in case of completely new shifts ($\gamma = 0$), pooled OLS and the anchor estimator achieve the same rate in $\gammaprime$, showcasing how finite robustness methods can perform similarly to empirical risk minimization if the assumptions on the robustness set are not met. We provide additional performance comparisons for the more general shift in \Cref{sec:apx-anchor-connections} and the proof of the corollary in \Cref{sec:apx-proof-of-corollary}.

\section{Experimental results}\label{sec:exp-results}
In this section, we provide empirical evidence of our theoretical conclusions in \Cref{sec:main-results-minimax,sec:comp-with-finite-robustness-methods}.
In particular, we compare the prediction performance of multiple existing robustness methods to the (estimated) minimax robustness in identifiable and partially identifiable settings.
We observe that both on synthetic and real-world data, in the partially identified setting, empirical risk minimization and invariance-based robustness methods not only have significantly sub-optimal test loss but also perform similarly, thereby aligning with our theoretical results in \Cref{sec:comp-with-finite-robustness-methods}. This stands in contrast to the identifiable setting, where the anchor predictor is optimal up to finite-sample effects. Furthermore, we observe that even though the minimizer of the worst-case robust risk is optimal only for the linear causal setting in \Cref{sec:training-data},  it surprisingly outperforms existing methods in a real-world experiment.
\label{sec:experiments}
\paragraph{Experiments on synthetic Gaussian data}
We simulate Gaussian covariates according to \cref{eqn:SCM} with multiple environments differing by linearly independent randomly selected mean shifts. For a randomly sampled collection of mean shifts, we evaluate a proxy for the worst-case robust risk by picking the most adversarial $(\betastar,\noisecovxxstar)$ for the shifts, and then computing its robust risk \eqref{eqn:robust-risk}.
We describe the full details of the data generation and loss evaluation in \cref{sec:apx-synthetic-exps}.
We consider two shift scenarios: in the 
identifiable case (see \Cref{fig:sub1-identified}), the test environment is only perturbed by bounded shifts in training directions with increasing strength $\gamma$, as considered in prior work \citep{rothenhausler2021anchor,shen2023causalityoriented}. In the 
non-identifiable case (see \Cref{fig:sub2-nonidentified}), the test environment is perturbed by a mixture of training shifts and shifts in previously unobserved directions, where $\gamma$ is fixed and $\gamma'$ varies (cf. Assumption~\ref{as:Mtest-structure}). 
We compute the empirical minimizers $\betaOLShat,\betaahat$ and $\betarobpihat$ of the OLS, anchor regression and worst-case robust losses, respectively, and compare their worst-case robust risk (mean squared error) in Figure~\ref{fig:synthetic-experiments}. In the identifiable setting -- \cref{fig:synthetic-experiments} (left) --  he robust risk is asymptotically constant across $\gamma$ for both robust methods, while the error for the OLS, or vanilla ERM, estimator increases linearly. In contrast, in the second, partially identified, setting -- \cref{fig:synthetic-experiments} (right) -- all estimators exhibit linearly increasing test errors; however the slopes of the anchor and OLS estimator
are much steeper and lead to larger errors than the empirical minimizer of \eqref{eqn:ID-robust-risk-formula} that 
closely matches the analytic theoretical 
lower bound.
\begin{figure}[ht]
\centering
\includegraphics[width=\textwidth]{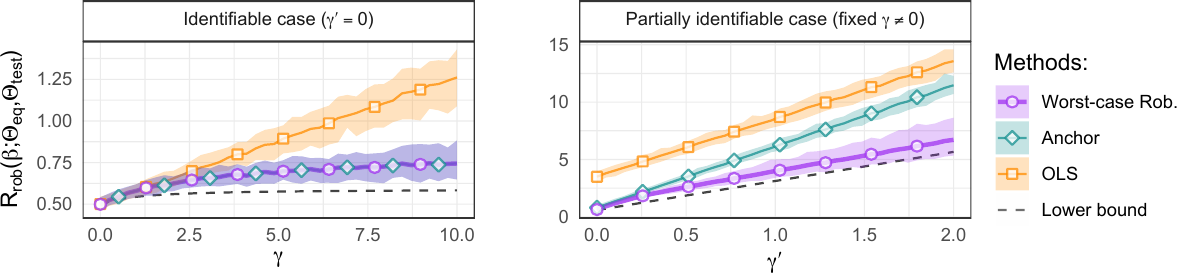}
\caption{\small{Worst-case robust risk 
of the baseline estimators $\betaOLS, \betaa$ (using the "correct" $\gamma$),  
the worst-case robust predictor in (mean-shifted) multi-environment finite-sample experiments and theoretical population lower bound in the classical identified setting with varying shift strength $\gamma$ (left) and the partially identifiable setting with fixed $\gamma$ but varying $\gamma'$ (right).  The details of the experimental setting can be found in \Cref{sec:apx-experiments}}. }
\label{fig:synthetic-experiments}
\end{figure}
\paragraph{Real-world data experiments} 
We evaluate the performance of OOD methods using single-cell gene expression data from \cite{replogle2022mapping}, consisting of $d = 622$ genes across observational and interventional environments.
As in \citep{schultheiss2024assessing}, we focus on 28 genes active in the observational environment.
For each gene $j = 1, \ldots, 28$, we define the target variable $Y := X_j$ and select the three genes most strongly correlated with $Y$ as covariates. This yields 28 prediction problems indexed by $j$,
each consisting of data from an observational environment $\mathcal{O}$ and three interventional environments $\mathcal{I}_{j1}$, $\mathcal{I}_{j2}$, $\mathcal{I}_{j3}$ representing the gene knockout on a single covariate.
For each prediction problem, we consider three training datasets $D_{j1}$, $D_{j2}$, $D_{j3}$, obtained by combining data from $\mathcal{O}$ with a single interventional environment $\mathcal{I}_{j1}$, $\mathcal{I}_{j2}$, $\mathcal{I}_{j3}$, respectively.
For each training dataset $D_{jk}$, $k = 1, 2, 3$, we evaluate the mean-squared error (MSE) at test time using four datasets consisting of varying proportions of unseen shifts (e.g., ``$33\%$ unseen directions'' in \Cref{fig:genes-res-weights} represents a test dataset with $67\%$ observations sampled from $\mathcal{I}_{jk}$ and $33\%$ from $\mathcal{I}_{j\ell}$ with $\ell \neq k$).
Hence, for each prediction problem predicting a gene $j$, we evaluate on
12 configurations (three training and four test datasets).\footnote{An illustration of the training and test setups can be found in \cref{fig:genes}.}
\Cref{fig:genes-res-weights} illustrates the test MSE of the worst-case robust estimator (Worst-case Rob.) alongside anchor regression, invariant causal prediction (ICP) \cite{peters2016causal}, DRIG, and OLS, as a function of perturbation strength 
$s$. \footnote{Details on the tuning parameter for each method are in \Cref{sec:apx-real-world}.}
For a given proportion of unseen shifts, $s$ controls the distance of the test data points from the observational mean, acting as a proxy for shift strengths $\gamma$ and $\gamma'$. \footnote{More details on the shift strength can be found in \Cref{sec:apx-real-world}.}
We observe that 
the performance ranking of the robustness methods significantly
varies with the proportion of new test shift directions. As expected, when no new shift 
directions are present at test time (0\%), 
anchor regression and DRIG are optimal, 
since they protect against shifts observed at training time. However, as soon as some unseen directions are present, their performance
becomes inferior to OLS/ERM and the gap to the worst-case robust predictor (in the linear setting described in \Cref{sec:setting})
grows with the proportion of unseen shifts.
Further, while the MSE of the previous invariant methods increases significantly with 
the strength of the test shift $s$,
the test loss of the worst-case robust predictor remains relatively stable.

\begin{figure}[ht]
    \centering
    \includegraphics[width=1\textwidth]{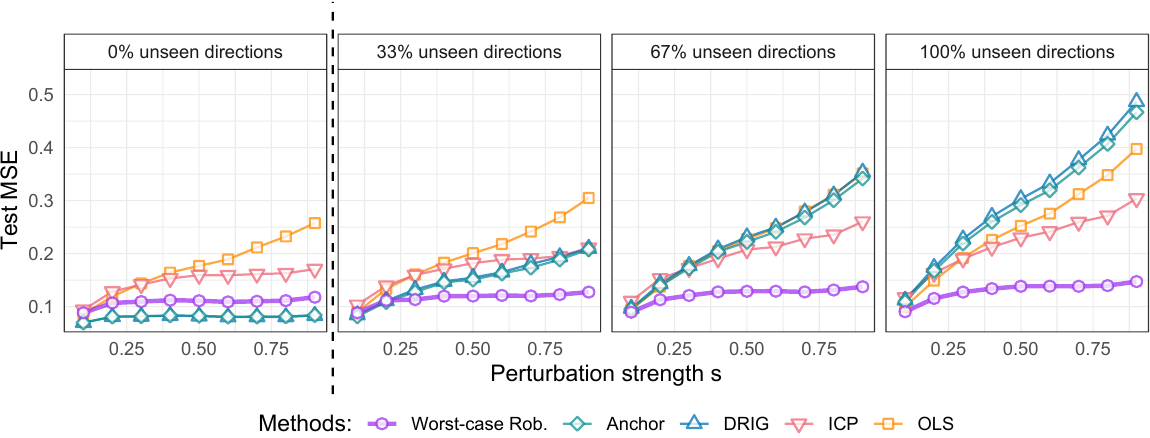}
    \caption{
    \small{ 
    The figures show the performance of the \emph{worst-case robust predictor} (Worst-case Rob.) compared to other methods as a function of perturbation strength $s$.
    Different panels correspond to the proportion of unseen shift directions at test time.
    For each panel and perturbation strength $s$, each point represents an average over the 28 target genes and three experiments (i.e., training environments).}
    }
    \label{fig:genes-res-weights}
\end{figure}
  
  \section{Conclusion and future directions}
  \label{sec:conclusion}
  This paper introduces the
\idRRs – a quantity that is well-defined
even in settings where the usual robust risk is not computable from training distributions,
and in identifiable scenarios \citep{rothenhausler2021anchor,shen2023causalityoriented} reduces to the conventional robust risk.
We instantiate our general framework for linear models with additive distribution shifts and compute tight lower bounds for this setting.
Further, we demonstrate how i) the benefits of invariance-based robustness methods strongly decrease in the partially identifiable setting; and ii) this suboptimality increases with perturbation strength and proportion of previously unobserved test shifts.

The main limitation of our paper is its reliance on a linear setting to explicitly compute the \idRRs and the corresponding minimax quantity. However, we expect that the results and intuition developed in this paper  can be extended to linear shifts in a 
lower-dimensional latent space via a suitable parametric or non-linear map \citep{thams2022evaluating, buchholz2024learning}. Important future directions include extending our results to more general shift models, non-linear functional relationships and the classification setting. 
Further, a potential use of our work is in the field of \emph{active intervention selection} (e.g, \citep{zhang2023active, gamella2020active}). By computing the most adversarial model parameter for a given estimator, e.g., OLS, we can obtain an intervention which minimizes the \idRRs of the estimator on the next unseen shift.

\section{Acknowledgements}
\label{sec:acknowledgements}
JK was supported by the SNF grant number 204439. NG was supported by the SNF grant number 210976. We thank Kasra Jalaldoust and Yixin Wang for helpful discussions and feedback on the manuscript. This work was done in part while JK and FY were visiting the Simons Institute for the Theory of Computing.
  
  \bibliography{biblio}

\begin{thebibliography}{10}

\bibitem{bental2013robust}
Aharon Ben-Tal, Dick den Hertog, Anja~De Waegenaere, Bertrand Melenberg, and
  Gijs Rennen.
\newblock Robust solutions of optimization problems affected by uncertain
  probabilities.
\newblock {\em Management Science}, 59(2):341--357, 2013.

\bibitem{duchi2021learning}
John~C Duchi and Hongseok Namkoong.
\newblock Learning models with uniform performance via distributionally robust
  optimization.
\newblock {\em The Annals of Statistics}, 49(3):1378--1406, 2021.

\bibitem{goodfellow2014explaining}
Ian~J. Goodfellow, Jonathon Shlens, and Christian Szegedy.
\newblock Explaining and harnessing adversarial examples.
\newblock In {\em International Conference on Learning Representations}, 2015.

\bibitem{madry2018towards}
Aleksander Madry, Aleksandar Makelov, Ludwig Schmidt, Dimitris Tsipras, and
  Adrian Vladu.
\newblock Towards deep learning models resistant to adversarial attacks.
\newblock In {\em International Conference on Learning Representations}, 2018.

\bibitem{mansour2008}
Yishay Mansour, Mehryar Mohri, and Afshin Rostamizadeh.
\newblock Domain adaptation with multiple sources.
\newblock In {\em Advances in Neural Information Processing Systems}, 2008.

\bibitem{meinshausen2014}
Nicolai Meinshausen and Peter Bühlmann.
\newblock Maximin effects in inhomogeneous large-scale data.
\newblock {\em The Annals of Statistics}, 43(4):1801--1830, 2015.

\bibitem{sagawa2019distributionally}
Shiori Sagawa*, Pang~Wei Koh*, Tatsunori~B. Hashimoto, and Percy Liang.
\newblock Distributionally robust neural networks.
\newblock In {\em International Conference on Learning Representations}, 2020.

\bibitem{buhlmann2020invariance}
Peter B{\"u}hlmann.
\newblock Invariance, causality and robustness.
\newblock {\em Statistical Science}, 35(3):404--426, 2020.

\bibitem{meinshausen2018causality}
Nicolai Meinshausen.
\newblock Causality from a distributional robustness point of view.
\newblock In {\em 2018 IEEE Data Science Workshop (DSW)}, pages 6--10. IEEE,
  2018.

\bibitem{shen2023causalityoriented}
Xinwei Shen, Peter B{\"u}hlmann, and Armeen Taeb.
\newblock Causality-oriented robustness: exploiting general additive
  interventions.
\newblock {\em arXiv preprint arXiv:2307.10299}, 2023.

\bibitem{Tramer19}
Florian Tramèr and Dan Boneh.
\newblock Adversarial training and robustness for multiple perturbations.
\newblock In {\em Advances in Neural Information Processing Systems}, 2019.

\bibitem{kang19}
Daniel Kang, Yi~Sun, Tom Brown, Dan Hendrycks, and Jacob Steinhardt.
\newblock Transfer of adversarial robustness between perturbation types.
\newblock {\em arXiv preprint arXiv:1905.01034}, 2019.

\bibitem{peters2016causal}
Jonas Peters, Peter B{\"u}hlmann, and Nicolai Meinshausen.
\newblock Causal inference by using invariant prediction: identification and
  confidence intervals.
\newblock {\em Journal of the Royal Statistical Society Series B: Statistical
  Methodology}, 78(5):947--1012, 2016.

\bibitem{rojas2018invariant}
Mateo Rojas-Carulla, Bernhard Sch{\"o}lkopf, Richard Turner, and Jonas Peters.
\newblock Invariant models for causal transfer learning.
\newblock {\em The Journal of Machine Learning Research}, 19(1):1309--1342,
  2018.

\bibitem{arjovsky2020invariant}
Martin Arjovsky, L{\'e}on Bottou, Ishaan Gulrajani, and David Lopez-Paz.
\newblock Invariant risk minimization.
\newblock {\em arXiv preprint arXiv:1907.02893}, 2019.

\bibitem{krueger2021out}
David Krueger, Ethan Caballero, Joern-Henrik Jacobsen, Amy Zhang, Jonathan
  Binas, Dinghuai Zhang, Remi Le~Priol, and Aaron Courville.
\newblock Out-of-distribution generalization via risk extrapolation ({RE}x).
\newblock In {\em International Conference on Machine Learning}, pages
  5815--5826. PMLR, 2021.

\bibitem{ahuja2020empirical}
Kartik Ahuja, Jun Wang, Amit Dhurandhar, Karthikeyan Shanmugam, and Kush~R.
  Varshney.
\newblock Empirical or invariant risk minimization? {A} sample complexity
  perspective.
\newblock In {\em International Conference on Learning Representations}, 2021.

\bibitem{ahuja2020invariant}
Kartik Ahuja, Karthikeyan Shanmugam, Kush Varshney, and Amit Dhurandhar.
\newblock Invariant risk minimization games.
\newblock In {\em International Conference on Machine Learning}, pages
  145--155. PMLR, 2020.

\bibitem{kamath2021does}
Pritish Kamath, Akilesh Tangella, Danica Sutherland, and Nathan Srebro.
\newblock Does invariant risk minimization capture invariance?
\newblock In {\em International Conference on Artificial Intelligence and
  Statistics}, pages 4069--4077. PMLR, 2021.

\bibitem{rosenfeld2020risks}
Elan Rosenfeld, Pradeep~Kumar Ravikumar, and Andrej Risteski.
\newblock The risks of invariant risk minimization.
\newblock In {\em International Conference on Learning Representations}, 2021.

\bibitem{rothenhausler2021anchor}
Dominik Rothenh{\"a}usler, Nicolai Meinshausen, Peter B{\"u}hlmann, and Jonas
  Peters.
\newblock Anchor regression: Heterogeneous data meet causality.
\newblock {\em Journal of the Royal Statistical Society Series B: Statistical
  Methodology}, 83(2):215--246, 2021.

\bibitem{tamer2010partial}
Elie Tamer.
\newblock Partial identification in econometrics.
\newblock {\em Annu. Rev. Econ.}, 2(1):167--195, 2010.

\bibitem{frake2023perfect}
Justin Frake, Anthony Gibbs, Brent Goldfarb, Takuya Hiraiwa, Evan Starr, and
  Shotaro Yamaguchi.
\newblock From perfect to practical: Partial identification methods for causal
  inference in strategic management research.
\newblock {\em Available at SSRN 4228655}, 2023.

\bibitem{yu1997assouad}
Bin Yu.
\newblock Assouad, {F}ano, and le {C}am.
\newblock In {\em Festschrift for Lucien Le Cam: Research papers in probability
  and statistics}, pages 423--435. Springer, 1997.

\bibitem{saengkyongam2022exploiting}
Sorawit Saengkyongam, Leonard Henckel, Niklas Pfister, and Jonas Peters.
\newblock Exploiting independent instruments: Identification and distribution
  generalization.
\newblock In {\em International Conference on Machine Learning}, pages
  18935--18958. PMLR, 2022.

\bibitem{bellot2022partial}
Kasra Jalaldoust, Alexis Bellot, and Elias Bareinboim.
\newblock Partial transportability for domain generalization.
\newblock In {\em Advances in Neural Information Processing Systems}, 2024.

\bibitem{shimodaira2000improving}
Hidetoshi Shimodaira.
\newblock Improving predictive inference under covariate shift by weighting the
  log-likelihood function.
\newblock {\em Journal of statistical planning and inference}, 90(2):227--244,
  2000.

\bibitem{Dufour2010Econometrics}
Jean-Marie Dufour and Cheng Hsiao.
\newblock {\em Identification}, pages 65--77.
\newblock Palgrave Macmillan UK, London, 2010.

\bibitem{amemiya1985advanced}
Takeshi Amemiya.
\newblock {\em Advanced Econometrics}.
\newblock Harvard University Press, 1985.

\bibitem{bowden1990instrumental}
Roger~J Bowden and Darrell~A Turkington.
\newblock {\em Instrumental variables}.
\newblock Number~8. Cambridge University Press, 1990.

\bibitem{chen2021domain}
Yuansi Chen and Peter B{\"u}hlmann.
\newblock Domain adaptation under structural causal models.
\newblock {\em Journal of Machine Learning Research}, 22(261):1--80, 2021.

\bibitem{replogle2022mapping}
Joseph~M Replogle, Reuben~A Saunders, Angela~N Pogson, Jeffrey~A Hussmann,
  Alexander Lenail, Alina Guna, Lauren Mascibroda, Eric~J Wagner, Karen
  Adelman, Gila Lithwick-Yanai, et~al.
\newblock Mapping information-rich genotype-phenotype landscapes with
  genome-scale {P}erturb-seq.
\newblock {\em Cell}, 185(14):2559--2575, 2022.

\bibitem{schultheiss2024assessing}
Christoph Schultheiss and Peter B{\"u}hlmann.
\newblock Assessing the overall and partial causal well-specification of
  nonlinear additive noise models.
\newblock {\em Journal of Machine Learning Research}, 25(159):1--41, 2024.

\bibitem{thams2022evaluating}
Nikolaj Thams, Michael Oberst, and David Sontag.
\newblock Evaluating robustness to dataset shift via parametric robustness
  sets.
\newblock {\em Advances in Neural Information Processing Systems}, 2022.

\bibitem{buchholz2024learning}
Simon Buchholz, Goutham Rajendran, Elan Rosenfeld, Bryon Aragam, Bernhard
  Sch{\"o}lkopf, and Pradeep Ravikumar.
\newblock Learning linear causal representations from interventions under
  general nonlinear mixing.
\newblock {\em Advances in Neural Information Processing Systems}, 2024.

\bibitem{zhang2023active}
Jiaqi Zhang, Louis Cammarata, Chandler Squires, Themistoklis~P Sapsis, and
  Caroline Uhler.
\newblock Active learning for optimal intervention design in causal models.
\newblock {\em Nature Machine Intelligence}, 5(10):1066--1075, 2023.

\bibitem{gamella2020active}
Juan~L Gamella and Christina Heinze-Deml.
\newblock Active invariant causal prediction: Experiment selection through
  stability.
\newblock {\em Advances in Neural Information Processing Systems}, 2020.

\bibitem{sinha2017certifying}
Aman Sinha, Hongseok Namkoong, and John Duchi.
\newblock Certifiable distributional robustness with principled adversarial
  training.
\newblock In {\em International Conference on Learning Representations}, 2018.

\bibitem{mohajerin2018data}
Peyman Mohajerin~Esfahani and Daniel Kuhn.
\newblock Data-driven distributionally robust optimization using the
  {W}asserstein metric: Performance guarantees and tractable reformulations.
\newblock {\em Mathematical Programming}, 171(1):115--166, 2018.

\bibitem{fan2023environment}
Jianqing Fan, Cong Fang, Yihong Gu, and Tong Zhang.
\newblock Environment invariant linear least squares.
\newblock {\em arXiv preprint arXiv:2303.03092}, 2023.

\bibitem{magliacane2018domain}
Sara Magliacane, Thijs Van~Ommen, Tom Claassen, Stephan Bongers, Philip
  Versteeg, and Joris~M Mooij.
\newblock Domain adaptation by using causal inference to predict invariant
  conditional distributions.
\newblock {\em Advances in Neural Information Processing Systems}, 2018.

\bibitem{shi2021gradient}
Yuge Shi, Jeffrey Seely, Philip Torr, Siddharth N, Awni Hannun, Nicolas
  Usunier, and Gabriel Synnaeve.
\newblock Gradient matching for domain generalization.
\newblock In {\em International Conference on Learning Representations}, 2022.

\bibitem{xie2020risk}
Chuanlong Xie, Haotian Ye, Fei Chen, Yue Liu, Rui Sun, and Zhenguo Li.
\newblock Risk variance penalization.
\newblock {\em arXiv preprint arXiv:2006.07544}, 2020.

\bibitem{ahuja2021invariance}
Kartik Ahuja, Ethan Caballero, Dinghuai Zhang, Jean-Christophe Gagnon-Audet,
  Yoshua Bengio, Ioannis Mitliagkas, and Irina Rish.
\newblock Invariance principle meets information bottleneck for
  out-of-distribution generalization.
\newblock {\em Advances in Neural Information Processing Systems}, 2021.

\bibitem{jakobsen2022distributional}
Martin~Emil Jakobsen and Jonas Peters.
\newblock Distributional robustness of {K}-class estimators and the {PULSE}.
\newblock {\em The Econometrics Journal}, 25(2):404--432, 2022.

\bibitem{christiansen2021causal}
Rune Christiansen, Niklas Pfister, Martin~Emil Jakobsen, Nicola Gnecco, and
  Jonas Peters.
\newblock A causal framework for distribution generalization.
\newblock {\em IEEE Transactions on Pattern Analysis and Machine Intelligence},
  44(10):6614--6630, 2021.

\bibitem{kook2022distributional}
Lucas Kook, Beate Sick, and Peter B{\"u}hlmann.
\newblock Distributional anchor regression.
\newblock {\em Statistics and Computing}, 32(3):39, 2022.

\bibitem{frogner2019incorporating}
Charlie Frogner, Sebastian Claici, Edward Chien, and Justin Solomon.
\newblock Incorporating unlabeled data into distributionally robust learning.
\newblock {\em Journal of Machine Learning Research}, 22(56):1--46, 2021.

\bibitem{liu2022distributionally}
Jiashuo Liu, Jiayun Wu, Bo~Li, and Peng Cui.
\newblock Distributionally robust optimization with data geometry.
\newblock {\em Advances in Neural Information Processing Systems}, 2022.

\bibitem{pearl2009causality}
Judea Pearl.
\newblock {\em Causality: Models, Reasoning and Inference}.
\newblock Cambridge University Press, USA, 2nd edition, 2009.

\bibitem{angrist1996identification}
Joshua~D Angrist, Guido~W Imbens, and Donald~B Rubin.
\newblock Identification of causal effects using instrumental variables.
\newblock {\em Journal of the American Statistical Association},
  91(434):444--455, 1996.

\bibitem{hartford2017deep}
Jason Hartford, Greg Lewis, Kevin Leyton-Brown, and Matt Taddy.
\newblock Deep {IV}: A flexible approach for counterfactual prediction.
\newblock In {\em International Conference on Machine Learning}, pages
  1414--1423. PMLR, 2017.

\bibitem{singh2019kernel}
Rahul Singh, Maneesh Sahani, and Arthur Gretton.
\newblock Kernel instrumental variable regression.
\newblock {\em Advances in Neural Information Processing Systems}, 2019.

\bibitem{bennett2019deep}
Andrew Bennett, Nathan Kallus, and Tobias Schnabel.
\newblock Deep generalized method of moments for instrumental variable
  analysis.
\newblock {\em Advances in Neural Information Processing Systems}, 2019.

\bibitem{muandet2020dual}
Krikamol Muandet, Arash Mehrjou, Si~Kai Lee, and Anant Raj.
\newblock Dual instrumental variable regression.
\newblock {\em Advances in Neural Information Processing Systems}, 2020.

\bibitem{gulrajani2021in}
Ishaan Gulrajani and David Lopez-Paz.
\newblock In search of lost domain generalization.
\newblock In {\em International Conference on Learning Representations}, 2021.

\bibitem{peters2017elements}
Jonas Peters, Dominik Janzing, and Bernhard Sch{\"o}lkopf.
\newblock {\em Elements of causal inference: foundations and learning
  algorithms}.
\newblock 2017.

\bibitem{yu2015useful}
Yi~Yu, Tengyao Wang, and Richard~J Samworth.
\newblock A useful variant of the {D}avis--{K}ahan theorem for statisticians.
\newblock {\em Biometrika}, 102(2):315--323, 2015.

\bibitem{van2000asymptotic}
Aad~W Van~der Vaart.
\newblock {\em Asymptotic statistics}, volume~3.
\newblock Cambridge University Press, 2000.

\bibitem{chevalley2022causalbench}
Mathieu Chevalley, Yusuf Roohani, Arash Mehrjou, Jure Leskovec, and Patrick
  Schwab.
\newblock Causalbench: A large-scale benchmark for network inference from
  single-cell perturbation data.
\newblock {\em arXiv preprint arXiv:2210.17283}, 2022.

\bibitem{qi2013repurposing}
Lei~S Qi, Matthew~H Larson, Luke~A Gilbert, Jennifer~A Doudna, Jonathan~S
  Weissman, Adam~P Arkin, and Wendell~A Lim.
\newblock Repurposing {CRISPR} as an {RNA}-guided platform for
  sequence-specific control of gene expression.
\newblock {\em Cell}, 152(5):1173--1183, 2013.

\end{thebibliography}
  \bibliographystyle{unsrt}


\clearpage
\appendix
\hypersetup{
    colorlinks,
   linkcolor={pierCite},
    citecolor={pierCite},
    urlcolor={pierCite}
}
\section*{Appendix}
The following sections provide deferred discussions, proofs and experimental details.

\section{Extended related work}
\label{sec:apx-related_work}

\begin{table}[tbp]
    \centering
    \caption{Comparison of various distributional robustness frameworks and what kind of assumptions their analysis can account for (with an incomplete list of examples for each framework). 
    }\label{tab:rw}
    \vspace{1pt}
\resizebox{.8\columnwidth}{!}{%
\begin{tabular}{@{}cccc@{}}
\toprule
Framework accounts for~ &
  \begin{tabular}[c]{@{}c@{}}~bounded~~\\ shifts\end{tabular} & 
  \begin{tabular}[c]{@{}c@{}}partial identifiability of\\  ~~model parameters ~~\end{tabular} &
  \begin{tabular}[c]{@{}c@{}}partial identifiability of\\  ~~robustness set\end{tabular} \\ 
  \toprule
\begin{tabular}[c]{@{}c@{}}DRO\\ \citep{bental2013robust, duchi2021learning, sinha2017certifying, mohajerin2018data, sagawa2019distributionally} \end{tabular} & \cmark & $-$  & \xmark  \\ \midrule
\begin{tabular}[c]{@{}c@{}}Infinite robustness \\
\citep{peters2016causal, fan2023environment, magliacane2018domain, rojas2018invariant,arjovsky2020invariant, ahuja2020invariant,
shi2021gradient, 
xie2020risk, krueger2021out, ahuja2021invariance}\\

\end{tabular}
& \xmark  & \xmark & \xmark \\ \midrule
\begin{tabular}[c]
{@{}c@{}}Finite robustness \\
\citep{rothenhausler2021anchor, jakobsen2022distributional, christiansen2021causal, kook2022distributional, shen2023causalityoriented}
\end{tabular} &\cmark  & \cmark  & \xmark  \\ \midrule
\begin{tabular}[c]
{@{}c@{}}
Partially id. robustness\\
(this work)~~
\phantom{~}
\end{tabular} & \cmark & \cmark & \cmark \\ \bottomrule
\end{tabular}}
\end{table} 
To put our work into context,
first, we discuss relevant distributional robustness literature organized according to structural assumptions on the desired robustness set (see \cref{tab:rw}). Second, we summarize existing views on partial identifiability in the causality and econometrics literature and how our findings connect to their perspective.

\textbf{\textit{No structural assumptions on the shift.}}
\textbf{DRO:} Distributionally robust optimization (DRO) tackles the problem of domain generalization when the robustness set is a ball around the training distribution w.r.t. some probability distance measure, e.g., Wasserstein distance \cite{sinha2017certifying, mohajerin2018data} or $f$-divergences \citep{bental2013robust, duchi2021learning}.
Considering all test distributions in a discrepancy ball can lead to overly conservative predictions, and therefore, alternatives have been proposed in, e.g., the Group DRO literature \citep{sagawa2019distributionally,frogner2019incorporating, liu2022distributionally}.
However, these methods cannot protect against perturbations larger than the ones seen during training time and do not provide a clear interpretation of the perturbations class \cite{shen2023causalityoriented}.

\textbf{\textit{Structural assumptions on the shift.}}
Robustness from the lens of causality takes a step further, by assuming a structural causal model \cite{pearl2009causality} generating the observed data $(X, Y)$. 
\textbf{Infinite robustness methods:} 
The motivation of causal methods for robustness is that the causal function is worst-case optimal to predict the response under interventions of arbitrary direction and strength on the covariates \citep{meinshausen2018causality, buhlmann2020invariance}. For this reason, causal models achieve what we call \emph{infinite robustness}.
Depending on the assumptions of the SCM, there are different ways to achieve infinite robustness.
When there are no latent confounders, several works \citep{peters2016causal, fan2023environment, magliacane2018domain, rojas2018invariant, ahuja2020invariant, arjovsky2020invariant, shi2021gradient, xie2020risk, krueger2021out, ahuja2021invariance}  aim to identify the causal parents and achieve infinite robustness by exploiting the heterogeneity across training environments.
In the presence of latent confounders, it is possible to achieve infinite robustness by identifying the causal function with, e.g., the instrumental variable method \cite{angrist1996identification, hartford2017deep, singh2019kernel, bennett2019deep, muandet2020dual}.
There are several limitations to  \emph{infinitely robust} methods. First, the identifiability conditions of the causal parents and/or causal function are often challenging to verify in practice.
Second, ERM can outperform these methods when the interventions (read shifts) at test time are not arbitrarily strong or act directly on the response or latent variable \citep{ahuja2020empirical, gulrajani2021in}.
\textbf{Finite robustness methods:} 
In real data, shifts of arbitrary direction and strength in the covariates are unrealistic. Thus, different methods \citep{rothenhausler2021anchor, jakobsen2022distributional, kook2022distributional, shen2023causalityoriented, christiansen2021causal}
trade-off robustness against predictive power to achieve what we call \emph{finite robustness}. 
The main idea of finite robustness methods is to learn a function that is as predictive as possible while protecting against shifts up to some strength in the directions observed during training time.
These methods, however, only provide robustness guarantees that depend on the heterogeneity of the training data and do not offer insights into the limits of \emph{algorithm-independent robustness} under shifts in new directions.

\textit{\textbf{Partial identifiability}}:
The problem of identification is at the center of the causal and econometric literature \citep{peters2017elements, amemiya1985advanced}. It studies the conditions under which the (population) training distribution uniquely determines the causal parameters of the underlying SCM.
Often, the training distribution only offers partial information about the causal parameters and, therefore, determines a set of observational equivalent parameters. This setting is known as \emph{partial} or \emph{set identification} and is used in causality and econometrics to learn intervals within which the true causal parameter lies \citep{tamer2010partial}.
In this work, we borrow the notion of partial identification to study the problem of distributional robustness when the robustness set itself is only partially identified.

\section{Extension to the general additive shift setting}\label{sec:apx-extension}


We discuss how our setting changes when we relax the assumptions on the existence of the reference environment. We consider the data-generating process in \cref{eqn:SCM}, where $\Ecaltrain = [m]$, $m \in \mathbb{N}$. If no environment $e$ exists with $\mue = 0$ and $\Sigmae = 0$, we first pick an arbitrary distribution $\PrefXY$ as the reference environment\footnote{In practice, it is useful to pick a distribution with the smallest covariance, i.e. $\trace \Cov(\Xref) \leq \trace \Cov(\Xe)$ for all $e$.} .
We denote $\noisecovxx' := \noisecovxxstar + \Sigmaref$. 

First, we show we can express the space $\cS$ of training additive shift directions defined in \cref{eqn:def-S} in the general case. We center all distributions by $\muref$  to obtain centered distributions $\tilde{\Pr}$ that $\EE_{X\sim \tilde{P}_e}[X]=0$. With respect to the arbitrary reference environment, we now define
\begin{align*}
    \tilde{\cS} \coloneqq \range \left[\bigcup_{e \in \Ecaltrain} \left(\Sigmae - \Sigmaref + (\mue - \muref)(\mue - \muref)^\top \right)\right]\subset \R^d.
\end{align*}
We now consider test shifts with respect to the environment $\PrefXY$\footnote{In other words, we require that the test distribution is a shifted version of the (arbitrarily) chosen reference distribution.}. We define the test shift upper bound $\Mtest = \gamma \Mseen + \gammaprime R R^\top$, where $\range(\Mseen) \subset \cS$ and $\range(R) \subset \cSperp$. Again, we can decompose the parameter $\betastar$ as $\betastar = \betastarS + \betastarperp$.
The projection $\betastarS$ of the causal parameter onto the relative training shifts induces the following observationally equivalent parameters corresponding to the reference distribution:
\begin{align*}
    \thetastarS := (\betaS, \noisecovxx', \noisecovxyS, \noisecovyyS) = (\betastarS, \noisecovxx', \noisecovxystar + \noisecovxx' \betastarperp, \noisecovyystar + 2 \langle \noisecovxystar, \betastarperp\rangle + \langle \betastarperp, \noisecovxx' \betastarperp\rangle).
\end{align*}
Again, $\thetastarS$ can be identified from the training distributions and is referred to as the \emph{identified model parameters}. The following adapted version of \cref{prop:invariant-set} shows that assuming shifts on $\PrefXY$, the robust prediction model is only identifiable if the test shifts are in the direction of the relative training shifts:
\begin{proposition}[Identifiability of reference distribution parameters and robust prediction model]
\label{prop:invariant-set-general} Suppose that the set of training and test distributions is generated according to \Cref{eqn:SCM,eqn:testAbound}.
Then, $\thetacS$ is observationally equivalent
to $\thetastar$ and computable from training distributions.
Furthermore, it holds that
\begin{enumerate}[(a)]
 \item 
 the model parameters 
 generating the reference distribution can be identified up to the following \idset: 
\begin{align*}
   \Invset = \{ \betastarS + \alpha, \noisecovxx', \noisecovxyS - \noisecovxx' \alpha, \noisecovyyS - 2 \alpha^\top \noisecovxyS + \alpha^\top \noisecovxx' \alpha \colon \alpha \in \cSperp \}  \ni \thetastar 
\end{align*}
\item 
the robust prediction model $\betarob$ as defined in \cref{eqn:formula-robust-predictor} is identified up to the set
    \begin{align*}
      \betastarS + (\gamma \projM + \noisecovxx')^{-1} \noisecovxyS + \{ (\gamma \projM + \noisecovxx')^{-1} \alpha\colon \, \alpha \in \range(R)  \} \ni \betarob 
    \end{align*}
\end{enumerate}
\end{proposition}
The proof is analogous to \cref{sec:apx-proof-invariant-set}. A version of \cref{thm:pi-loss-lower-bound} for perturbations on the reference environment follows accordingly.


\section{Comparison to finite robustness methods continued}\label{sec:apx-anchor-connections}
\subsection{The setting of continuous anchor regression \cite{rothenhausler2021anchor}}\label{subsec:anchor}
In this section, we evaluate the \idRRs of the continuous anchor regression estimator.
    In the continuous anchor regression setting, during training we 
    observe the distribution 
    according to the process $X = M A + \eta$; $Y = {\betastar}^\top X + \xi$, where $A$ is an observed $q$-dimensional anchor variable with mean $0$ and covariance $\Sigma_A$ and $M \in \R^{d \times q}$ is a known matrix. Note that in this setting, we do not have a reference environment, but, since the anchor variable is observed, the distribution of the additive shift $M A$ is known.  The test shifts are assumed to be bounded by $\Mtest = \gamma M \Sigma_A M^\top$. Since $\range(\Mtest )\subset \cS = \range(M)$, no new directions are observed during test time, in other words, $R = 0$. Thus, both the corresponding robust loss and the anchor regression estimator can be determined from training data. It holds that
\begin{align*}
    \betaa = \argmin_{\beta \in \R^d} \Lossrob(\beta; \thetastar,\gamma M \Sigma_A M^\top).
\end{align*}
Again, the pooled OLS estimator corresponds to $\betaa$ with $\gamma = 1$. Similar to the discrete anchor case, in case the test shifts are given by $\Mnew = \gamma M \Sigma_A M^\top + \gammaprime R R^\top$, the worst-case robust risk \eqref{eqn:PI-robust-loss} is given by
\begin{align*}
    \Lossrobpi(\beta; \Invset, \Mnew) = \gammaprime (\Cker + \| R^\top \beta \|_2)^2 + \Lossrob(\beta; \thetastar,\gamma M \Sigma_A M^\top) 
\end{align*}
and for the best worst-case robustness of the anchor estimator it holds 
\begin{align*}
    \Lossrobpi(\betaa, \Invset; \Mtest)&= (\Cker + \| R R^\top (\noisecovxxstar + \gamma M \Sigma_A M^\top )^{-1} \noisecovxyS \| )^2 \gammaprime + \text{const}; \\
    \lim_{\gammaprime \to 0} \Lossrobpi(\betaa, \Invset;\Mnew)/\gammaprime &= \lim_{\gammaprime \to 0} \minimaxPIlossarg{\Mnew}/ {\gammaprime}.
\end{align*}
The above results follow by analogy with \cref{sec:apx-proof-of-corollary}.

\subsection{Distributionally robust invariant gradients (DRIG) \cite{shen2023causalityoriented}}\label{subsec:drig}
DRIG \cite{shen2023causalityoriented} introduce a more general additive shift framework, where a collection of additive shifts $\Ae$ is given with moments $(\mue,\Sigmae)$. For each environment $e$, we observe data $(\Xe, \Ye)$ distributed according to the equations $\Xe= \Ae + \eta; \: \Ye = {\betastar}^\top \Xe + \xi$, where the noise is distributed like in \cref{eqn:SCM}. This DGP arises from the structural causal model assumption as described in \cref{fig:ex-scm}.  DRIG consider more a more general intervention setting, additionally allowing additive shifts of $Y$ and hidden confounders $H$. However, their identifiability results can only be shown for the case of interventions on $X$, and since identifiability of the causal parameter is a crucial part of our analysis, we only consider shifts on the covariates. DRIG assumes existence of a reference environment  $e = 0$ with $\mu_0 = 0$ and for which it is required that the second moment of the reference environment is dominated by the second moment of the training mixture: 
\begin{align*}
    \Sigma_0 \preceq \sum_{e \in [m]} w_e (\Sigma_e + \mu_e \mu_e^\top).
\end{align*}
This assumption allows \cite{shen2023causalityoriented} to derive the DRIG estimator which is robust against test shifts upper bounded by $\Mdrig :=  \gamma \sum_{e \in [m]} w_e (\Sigma_e - \Sigma_0 + \mu_e \mu_e^\top)$. The following lemma allows us to make further statements about $\Mdrig$:
\begin{lemma}\label{lm:span-inclusion}
    Let $A$ and $B$ be positive semidefinite matrices such that $B \preceq A$. Then it holds that $\range(B) \subset \range(A)$. 
\end{lemma}
\begin{proof}
    It suffices to show that $\ker(A) \subset \ker(B)$. ($\ker(A) \subset \ker (B)$ implies that $\range(A) = \ker (A)^\perp  \subset \ker(B)^\perp = \range(B)  $.) Consider $x \in \ker(A)$, $x \neq 0$. Then it holds that $x^\top (A - B) x = x^\top A x - x^\top B x = 0 - x^\top B x \geq 0$, from which it follows that $x^\top B x = 0$ and thus $x \in \ker(B)$. 
\end{proof}
Because of the assumption $\Sigma_0 \preceq \sum_{e \in [m]} w_e (\Sigma_e + \mue \mue^\top)$, by \cref{lm:span-inclusion} it follows that $\range(\Sigma_0) \subset \range \sum_{e \geq 1} ( \Sigma_e + \mu_e \mu_e^\top )$  and thus 
\begin{align*}
\range(\Mdrig) \subseteq \range\left( \sum_{e \geq 1} w_e (\Sigma_e + \mu_e \mu_e^\top) \right). 
\end{align*}
Hence, the robustness directions achievable by DRIG in the "dominated reference environment" setting are the same as the ones under the assumption $\Sigma_0 = 0$. \\
Again, we observe that the test shifts bounded by $\gamma \Mdrig$ are fully contained in the space of identified directions $\cS$. If the test shifts are instead bounded by $\Mnew := \gamma \Mdrig + \gammaprime R R^\top$,  including some unseen directions $\range(R) \subset \cSperp$, the robust risk in the DRIG setting is only partially identified. The worst-case robust risk \eqref{eqn:PI-robust-loss} is given by 
\begin{align*}
    \Lossrobpi(\beta; \Invset, \Mnew) = \gammaprime (\Cker + \| R^\top \beta \|_2)^2 + \Lossrob(\beta; \thetastar, \gamma \Mdrig),
\end{align*}
and again, the DRIG estimator is optimal for infinitesimal shifts $\gammaprime$ and suboptimal for larger $\gammaprime$:
\begin{equation*}
    \begin{aligned}
        \Lossrobpi(\betaDRIG; \Invset,\Mnew) &= (\Cker + \| R R^\top (\noisecovxxstar + \gamma \Mdrig)^{-1} \noisecovxyS \| )^2 \gammaprime + \text{const}; \\ 
\text{whereas }\frac{\minimaxPIlossarg{\Mnew}}{\gammaprime} &= \Cker^2, \: \text{ if } \gammaprime \geq \gammath; \\  \lim_{\gammaprime \to 0} \frac{\minimaxPIlossarg{\Mnew}}{\gammaprime} &= (\Cker + \| R R^\top (\noisecovxxstar + \gamma \Mdrig)^{-1} \noisecovxyS\| )^2.
    \end{aligned}
\end{equation*}
The above results follow by plugging $\Mnew$ with $M := \Mdrig$ into the proof of \cref{cor:estimators} in \cref{sec:apx-proof-of-corollary}.

\section{Empirical estimation of the worst-case robust predictor}\label{sec:apx-empirical-estimation}

In this section, we discuss how to compute the worst-case robust loss and its minimizer from finite-sample multi-environment training data. We first describe the finite-sample setting and provide a high-level algorithm. We then discuss some parts of the algorithm in more detail. Finally, we show that the empirical worst-case robust loss is consistent under certain assumptions.  
Recall that we assume that $\Mtest = \gamma \PSMpop + \gammaprime R R^\top$, where $\gamma, \gammaprime \geq 0$, $\PSMpop $ is a PSD matrix satisfying $\range (\PSMpop) \subset \cS$ and $R$ is a semi-orthogonal matrix satisfying $\range (R) \subset \cSperp$. 

\subsection{Computing the worst-case robust loss}

\begin{algorithm}[h!]
    \caption{Computation of the worst-case robust loss} \label{alg:id-rob-loss}
    \begin{algorithmic}[1]
    \State \textbf{Input:} Multi-environment data $\cD \coloneqq \cup_{e \in \Ecaltrain} \cD_e$, test shift strengths $\gamma, \gammaprime > 0$, test shift directions $\Mtest \in \R^{d \times d}$, upper bound $C > 0$ on the norm of $\betastar$.
    
    \State \textbf{Step 1:} Estimate the training shift directions $\cShat(\cD)$, its  orthogonal complement $\cSperphat(\cD)$, and the identified linear parameter $\betaShat$.
    \State \textbf{Step 2:} Estimate the identified and non-identified test shift upper bounds $\Mseenemp$, $\Rhat \Rhat^\top$, respectively, from $\Mtest$, $\cShat(\cD)$ and $\cSperphat(\cD)$.
    \State \textbf{Step 3:} Estimate the maximum norm $\Ckerhat$ of the non-identified linear parameter.

    \State \vphantom{$\hat{R}$}\textbf{Step 4:} Compute the worst-case robust loss function 
    \begin{align*}
        \LL_n(\beta; \betaShat, \PSM, \PSperpM) &\gets \underbrace{\LLref(\beta; \cD_0)}_{\text{reference loss}} + \underbrace{\LLinv(\beta; \betaShat, \PSM, \gamma)}_{\text{invariance penalty term}} + \underbrace{\LLid(\beta; \Ckerhat, \RRhat, \gammaprime)}_{\text{non-identifiability penalty term}}.
    \end{align*}
    \State \textbf{Return:}  worst-case robust predictor and the estimated minimax "hardness" of the problem: 
    \begin{align*}
        \betarobpihat &\gets \argmin_{\beta \in \R^d} \LL_n(\beta; \betaShat, \PSM, \PSperpM); \\ 
       \hat{\mathfrak{M}}(\cD,\gamma, \gammaprime, \Mtest) &\gets \min_{\beta \in \R^d} \LL_n(\beta; \betaShat, \PSM, \PSperpM).
    \end{align*}

    \end{algorithmic}
\end{algorithm}

\paragraph{Training data.} We observe data from $m + 1$ training environments indexed by $E \in \Ecaltrain = \{0,..., m\}$, where $E = 0$ represents the reference environment. We impose a discrete probability distribution $\prob^{E}$ on the training environment $E \in \Ecaltrain$, resulting in the joint distribution $(X, Y, E) \sim \prob^{X, Y \mid E} \times \prob^{E}$.  For each environment $E = e$, we observe the samples $\cD_e \coloneqq \{(X_{e,i}, Y_{e,i})\}_{i = 1}^{n_e}$, where $(X_{e, i}, Y_{e, i})$ are independent copies of $(X_e, Y_e) \sim \prob^{X, Y \mid E = e}$. Then, the resulting dataset is $\cD \coloneqq \cup_{e \in \Ecaltrain} \cD_e$ with $n \coloneqq n_0 + \cdots + n_m$. Furthermore, for each environment $E = e$, we define the weights $w_e \coloneqq n_e / n$. 

\paragraph{Computation of the worst-case robust loss.} In \cref{alg:id-rob-loss}, we present a high-level scheme for computing the worst-case robust loss from multi-environment data, which consists of multiple steps. First, nuisance parameters related to the training and test shift directions are estimated, which we describe in more detail below. Afterwards, the three terms of the loss are computed: the (squared) loss $\LLref(\beta; \cD_0)$ on the reference environment is computed as 
\begin{align*}
       \LLref(\beta; \cD_0) = \sum_{i = 1}^{n_0} (Y_{0,i} - \beta^\top X_{0,i})^2. 
    \end{align*}
The invariance penalty term $\LLinv(\beta; \betaShat, \PSM, \gamma)$ (which increasingly aligns any estimator $\beta$ in the direction of the estimated invariant predictor $\betaShat$ as $\gamma \to \infty$) can be computed as following in the linear setting:
    \begin{align*}
       \LLinv(\beta; \betaShat, \PSM, \gamma) = \gamma  (\beta - \betaShat)^\top \PSM (\beta - \betaShat) . 
    \end{align*}
Finally, the non-identifiability penalty term $\LLid(\beta; \Ckerhat, \PSperpM, \gamma)$ can be computed as follows:
    \begin{align*}
       \LLid(\beta; \Ckerhat, \PSperpM, \gammaprime) &= \gammaprime (\Cker + \| \PSperpM \beta \|_2)^2. 
    \end{align*}
The non-identifiability term, with increasing $\gammaprime$, penalizes any predictor $\beta$ towards zero on the subspace $\Rhat$ of non-identified test shift directions. In total, the worst-case robust loss (in the linear setting) equals
\begin{align*}
    \LL_n(\beta; \betaShat, \PSM, \PSperpM) = \sum_{i = 1}^{n_0} (Y_{0,i} - \beta^\top X_{0,i})^2 + \gamma (\beta - \betaShat)^\top \PSM (\beta - \betaShat) + \gammaprime (\Cker + \| \PSperpM \beta \|_2)^2, 
\end{align*}
where we suppress dependence on $C$, $\gamma$ and $\gammaprime$ and only leave the dependence on the nuisance parameters.
\paragraph{Choice/Estimation of nuisance parameters.} We now provide more details on the empirical estimation of the nuisance parameters $\cShat, \cSperphat, \Rhat$, $\PSM$, and $\betaShat$. 
\begin{itemize}
    \item The \textbf{constant} $C$ corresponds to the upper bound on the norm of the true causal parameter $\betastar$. Thus, the practitioner chooses $C$ in advance to ensure that (with high probability) $\| \betastar \|_2 \leq C$. 
    \item The \textbf{training shift directions} $\cShat$ can be computed via 
     \begin{align}\label{eq:S-estimation}
        \cShat(\cD) = \mathrm{range} \left[\sum_{e = 1}^m ( \Cov(X^e)- \Cov(X^0) + \mu_e \mu_e^\top - \mu_0 \mu_0^\top)\right],
    \end{align}
where for $e \in \Ecaltrain$, the matrix $\Cov(X^e)$ is the empirical covariance matrix estimated within the training environment $E = e$, and $\mu_e \in \R^d$ is the empirical mean of the covariates within the training environment $E = e$. Additionally, we compute the orthogonal complement $\cSperphat(\cD)$ of the space $\cShat(\cD)$\footnote{In general, $S(\cD)$ is a proper subspace of $\R^d$ and the RHS of \eqref{eq:S-estimation} corresponds to a sum of low-rank second moments. This can be consistently estimated if, for instance, the rank of each shift is known (e.g. in the mean shift setting), or the covariances have a spiked structure, allowing to cut off small eigenvalues.}. 
\item  Computation of the \textbf{seen and unseen test shift directions.} Multiple options are possible for the practitioner to compute the empirical test shift directions $\Mseenemp$ and $\Rhat \Rhat^\top$. One option is to choose $\Mseenemp = \sum_{e} w_e ( \Cov(X^e)- \Cov(X^0) + \mu_e \mu_e^\top - \mu_0 \mu_0^\top)$, where $w_e$ is the proportions of observations in environment $e \in \Ecaltrain$, akin to anchor regression \cite{rothenhausler2021anchor} and DRIG \cite{shen2023causalityoriented} with an appropriate shift magnitude $\gamma$. Afterwards, $\Rhat \Rhat^\top$ is chosen to be a projection onto an appropriate subspace of $\cSperphat$ (if additional information about test shift directions is available). If no information is given, we can choose $\Rhat \Rhat^\top = \Pi_{\cSperphat}$. Alternatively, if the information about potential test shift directions is given in form of a PSD matrix $\Mtest$, for instance, $\Mtest$ being a projection onto some subspace, we can decompose $\Mtest$ 
into identified and non-identified shift directions (and their corresponding projection matrices). 
Let $\PicShat$ and $\PicSperphat$ denote the projection matrices on $\cShat(\cD)$ and $\cSperphat(\cD)$, respectively. Consider the singular value decompositions $\PicShat \Mtest = U_{\cShat} \Sigma_{\cShat} V_{\cShat}^\top$ and $\PicSperphat \Mtest = U_{\cSperphat} \Sigma_{\cSperphat} V_{\cSperphat}^\top$ Then,  define
\begin{align*}
\Shat &= U_{\cShat}, \quad \Rhat = U_{\cSperphat}.
\end{align*}
The subspaces $\range (\PicShat \Mtest)$ and $\range( \PicSperphat \Mtest)$ are minimal subspaces contained in $\cShat$ and $\cSperphat$, respectively, such that $\range(\Mtest) \subset \range( \PicShat \Mtest) \oplus \range( \PicSperphat \Mtest)$. We can then take as $\PSM$ and $\PSperpM$ their corresponding projection matrices. 
\item The \textbf{identified parameter} $\betaShat$ (approximately) equals the true invariant parameter $\betastar$ on the space of training shift directions $\cShat$. 
As conjectured in the anchor regression literature \citep{rothenhausler2021anchor,shen2023causalityoriented, jakobsen2022distributional} (see, for example, the discussion right after Theorem~3.4 in \citep{jakobsen2022distributional} and Appendix~H.3 therein)
for $\gamma \to \infty$, the estimators $\beta^{\gamma}_{\mathrm{anchor}}$ and $\beta^{\gamma}_{\mathrm{DRIG}}$ converge to the invariant parameter $\betastar$ on $\cS$. Thus, the identified parameter can be estimated as
\begin{align*}
    \betaShat \coloneqq \PicShat \beta^{\infty}_{\mathrm{anchor}} \quad \text{or} \quad \betaShat \coloneqq \PicShat \beta^{\infty}_{\mathrm{DRIG}}
\end{align*}
for the setting of mean or mean+variance shifts, respectively. 
\end{itemize}

\subsection{Consistency of the worst-case robust predictor} 
For any estimator $\beta \in \R^d$ and given the estimated nuisance parameters $\hat\varphi \coloneqq (\PSM, \PSperpM, \betaShat)$, we define the sample worst-case robust risk as
\begin{align}\label{eqn:rob-loss-ell-sample}
\begin{split}
    \LL_n(\beta, \hat{\varphi})  \coloneqq &
    \frac{1}{n_0}
    \sum_{i \in \cD_0}\left(Y_{0,i} - \beta^\top X_{0,i} \right)^2
    + \gamma (\betaShat - \beta)^\top \PSM (\betaShat - \beta) 
    +   \gammaprime \left(\sqrt{C - \norm{\betaShat}_2^2} 
+ \norm{\PSperpM\beta}_2\right)^2.
\end{split}
\end{align}
Correspondingly, we define the estimator of the worst-case robust predictor by
\begin{align}\label{eqn:betarobpi-consistent}
    \betarobpihat \coloneqq \argmin_{\beta \in \mathcal{B}} \LL_n(\beta, \hat{\varphi}),
\end{align}
where $\mathcal{B} \subseteq \R^d$ is some compact set whose interior contains the true parameter $\betarobpi$.

To show the consistency of~\eqref{eqn:betarobpi-consistent}, we first require consistency of the nuisance parameter estimators, which we state as an assumption.
\begin{assumption}\label{ass:consistency-nuisance}
    The estimated nuisance parameters $\hat\varphi \coloneqq (\PSM, \PSperpM, \betaShat)$ are consistent, that is, for $n \to \infty$,
    \begin{align*}
       \norm{\PSM - \PSMpop}_F \stackrel{\prob}{\to}0,
        \quad
        \norm{\PSperpM - \PSperpMpop}_F \stackrel{\prob}{\to}0,
        \quad
        \hat\beta^\cS \stackrel{\prob}{\to} \betastarS \coloneqq \Pi_{\cS} \betastar,
    \end{align*}
    where for any matrix $A \in \R^{m \times q}$, $\norm{A}_F = \sqrt{\trace(A^\top A)}$ denotes the Frobenius norm,  $\PSMpop$ is a PSD matrix with bounded eigenvalues with $\range(\PSMpop) \subset \cS$, and $\PSperpMpop$ is the corresponding population projection matrix onto $\cSperp$. 
\end{assumption}
 Depending on the assumptions of the data-generating process, Assumption~\ref{ass:consistency-nuisance} can be shown to hold. For example, in the anchor regression setting \cite{rothenhausler2021anchor}, the consistency of $\Mseen = \Manchor$, the projection
matrix $\PSperpM$, and $\PicShat$ holds if the dimension of $\cS$ is known (due to the mean shift structure).
The proof relies on the Davis--Kahan theorem (see, for example, \citep{yu2015useful}) and the consistency of the covariance matrix estimator.
Moreover, in the anchor regression setting, it is conjectured that the estimator $\beta^{\infty}_{\mathrm{anchor}}$ converges to its population counterpart (as discussed right after Theorem~3.4 in \citep{jakobsen2022distributional} and Appendix~H.3 therein) which implies that $\betaShat \coloneqq \PicShat \beta^{\infty}_{\mathrm{anchor}} $ consistently estimates $\betastarS = \Pi_{\cS} \betastar$.

Under the assumption of the consistency of the nuisance parameter estimators, we can now show that~\eqref{eqn:betarobpi-consistent} is a consistent estimator of the worst-case robust predictor.

\begin{proposition}\label{prop:consistency-predictor}
  Consider the estimator  $\betarobpihat$ of the worst-case robust predictor defined in~\eqref{eqn:betarobpi-consistent}. Suppose the optimization problem is over a compact set $\cB \subseteq \R^d$ whose interior contains the true minimizer $\betarobpi$. 
  Assume that the covariance matrix $\EE[X_0X_0^\top] \succ 0$ with bounded eigenvalues and $\EE[Y_0^2] < \infty$.
  Then, 
  under Assumption~\ref{ass:consistency-nuisance},
  $\betarobpihat$ is a consistent estimator of~$\betarobpi$.
\end{proposition}

\subsection{Proof of \Cref{prop:consistency-predictor}}

    For ease of notation define $\beta_0 \coloneqq \betarobpi$ and $\hat{\beta} \coloneqq \betarobpihat$.
    For any parameter of interest $\beta \in \cB$ and nuisance parameters $\varphi = (P_S, P_R, b)$,
   define the function 
  \begin{align}\label{eq:g-func}
      (x, y) \mapsto g_{\beta,\varphi}(x, y) 
      \coloneqq (y - \beta^\top x)^2 
      + \gamma \norm{P_S^{1/2}(b - \beta)}_2^2
      + \gamma \left(\sqrt{C - \norm{b}_2^2}  + \norm{P_R \beta}_2\right)^2.
  \end{align}
  Using~\eqref{eq:g-func}, the robust identifiable risk and its sample version defined in~\eqref{eqn:rob-loss-ell-sample} can be written, respectively as
  \begin{align*}
      \LL(\beta, \varphi) = \EE[g_{\beta, \varphi}(X_0, Y_0)],
      \quad
      \LL_n(\beta, \varphi) = \frac{1}{n_0}\sum_{i \in \cD_0} g_{\beta, \varphi}(X_{0,i}, Y_{0,i}).
  \end{align*}
    Our goal is to show that $\hat{\beta} \stackrel{\prob}{\to}\beta_0$. First, we show that the minimum of the loss is well-separated.

    \begin{lemma}\label{lm:well-separation}
    Suppose that $\EE[X_0X_0^\top] \succ 0$.
    Then, for all $\delta > 0$, it holds that
\begin{align}\label{eqn:well-sep}
    \inf\left\{\LL(\beta, \varphi_0) \colon \norm{\beta - \beta_0}_2 > \delta \right\} > \LL(\beta_0, \varphi_0).
\end{align}
\end{lemma}
    
   Fix $\delta > 0$.  From the well-separation of the minimum from Lemma~\ref{lm:well-separation}, there exists $\varepsilon > 0$ such that 
    \begin{align*}
        \left\{\norm{\hat\beta - \beta_0}_2 > \delta\right\} \subseteq
        \left\{\LL(\hat\beta, \varphi_0)- \LL(\beta_0, \varphi_0) > \varepsilon\right\}.
    \end{align*}
    Therefore,
    \begin{align}
        \prob&\left(\norm{\hat\beta - \beta_0}_2 > \delta \right) \leq 
        \prob\left(\LL(\hat\beta, \varphi_0)- \LL(\beta_0, \varphi_0) > \varepsilon\right) \nonumber\\
        &\quad= 
        \prob\left( 
        \LL(\hat\beta, \varphi_0)- \LL_n(\hat\beta, \varphi_0) + \LL_n(\hat\beta, \varphi_0)-
        \LL_n(\hat\beta, \hat\varphi)
        \right. \nonumber\\
        &\quad\quad\quad\quad
        \left. 
        + \LL_n(\hat\beta, \hat\varphi)
        - \LL_n(\beta_0, \hat\varphi)
        + \LL_n(\beta_0, \hat\varphi)
        -
        \LL(\beta_0, \varphi_0) > \varepsilon
        \right) \nonumber\\
        &\quad \leq
        \prob\left(
        \LL(\hat\beta, \varphi_0)- \LL_n(\hat\beta, \varphi_0) > \varepsilon/4
        \right)
        + \prob\left(
        \LL_n(\hat\beta, \varphi_0)-
        \LL_n(\hat\beta, \hat\varphi) > \varepsilon/4
        \right) \label{eq:consistency-1}\\
        &\quad\quad +
        \prob\left(
        \LL_n(\hat\beta, \hat\varphi)
        - \LL_n(\beta_0, \hat\varphi) > \varepsilon / 4
        \right)
        + \prob\left(
        \LL_n(\beta_0, \hat\varphi)
        -
        \LL(\beta_0, \varphi_0)>\varepsilon/4
        \right) \label{eq:consistency-2}.
    \end{align}
    We now want to prove convergence the four terms in ~\eqref{eq:consistency-1} and~\eqref{eq:consistency-2}. For this, we use the following statements proved in \Cref{sec:auxlemmaproofs}.

\begin{lemma}\label{lm:ulln-2}
    Suppose $\cB \subseteq \R^d$ is a compact set.
    Moreover, assume that the covariance matrix $\EE[X_0X_0^\top] \succ 0$ with bounded eigenvalues and $\EE[Y_0^2] < \infty$.
    Then, as $n, n_0 \to \infty$ it holds that 
    \begin{align}\label{eqn:ulln-2}
            \sup_{\beta \in \cB} |\LL_n(\beta, \varphi_0) - \LL(\beta, \varphi_0)| \stackrel{\prob}{\to} 0.
        \end{align}
\end{lemma}
\begin{lemma}\label{lm:lipschitz}
    As $n \to \infty$, it holds that
    \begin{align}\label{eqn:lipschitz}
            \sup_{\beta\in\mathcal{B}}|\LL_n(\beta, \hat\varphi) - \LL_n(\beta, \varphi_0)| \stackrel{\prob}{\to} 0.
        \end{align}
\end{lemma}
The two terms in~\eqref{eq:consistency-1} converge to 0 by \cref{lm:ulln-2} and \cref{lm:lipschitz}, respectively. The first term in~\eqref{eq:consistency-2} equals 0 since $\hat\beta$ minimizes $\beta \mapsto \LL_n(\beta, \hat\varphi)$. 
Finally, we observe that \begin{align}\label{eqn:ulln-1}
        \sup_{\beta \in \cB} |\LL_n(\beta, \hat\varphi) - \LL(\beta, \varphi_0)| \stackrel{\prob}{\to} 0,
        \end{align}
since we have that
    \begin{align*}
        \sup_{\beta \in \cB} |\LL_n(\beta, \hat\varphi) - \LL(\beta, \varphi_0)|
        \leq &\
        \sup_{\beta \in \cB} |\LL_n(\beta, \hat\varphi) - \LL_n(\beta, \varphi_0)|
        +
        \sup_{\beta \in \cB} |\LL_n(\beta, \varphi_0) - \LL(\beta, \varphi_0)|,
    \end{align*}
    where the first term converges in probability by Lemma~\ref{lm:lipschitz}, and the second term converges in probability by Lemma~\ref{lm:ulln-2}. This implies that the second term in~\eqref{eq:consistency-2} converges to zero. Since $\delta > 0$ was arbitrary, it follows that $\hat{\beta} \stackrel{\prob}{\to}\beta_0$.

\subsection{Proof of auxiliary lemmas}
\label{sec:auxlemmaproofs}
\subsubsection{Proof of \Cref{lm:well-separation}}
By definition,
\begin{align*}
    \LL(\beta, \varphi_0) = \EE[(Y_0 - \beta^\top X_0)^2]
    + \gamma \norm{\PSMpop^{1/2}(\betastarS - \beta)}_2^2
      + \gamma \left(\sqrt{C - \norm{\betastarS}_2^2}  + \norm{\PSperpMpop \beta}_2\right)^2.
\end{align*}
Since $\EE[X_0X_0^\top] \succ 0$, the first term is strongly convex in $\beta$.
Moreover, the second and third terms are convex in $\beta$. Therefore, $\LL(\beta, \varphi_0)$ is strongly convex in $\beta$. Since $\LL(\beta, \varphi_0)$ is also continuous in $\beta$, it follows that there exists a unique global minimum. Let $\beta_0$ denote the global minimizer of $\LL(\beta, \varphi_0)$.
By the fact that $\LL(\beta_0, \varphi_0)$ is a global  minimum, and by definition of strong convexity, there exists a positive constant $m > 0$ such that, for all $\beta \in \cB$,
\begin{align}\label{eq:strong-conv}
    \LL(\beta, \varphi_0) \geq \LL(\beta_0, \varphi_0) + \frac{m}{2} \norm{\beta - \beta_0}_2^2.
\end{align}
Fix $\delta > 0$.  
Then, by~\eqref{eq:strong-conv},  for all $\beta \in \cB$ such that $\norm{\beta-\beta_0}_2 > \delta$ it holds that
\begin{align*}
    \LL(\beta, \varphi_0) \geq \LL(\beta_0, \varphi_0) + \frac{m\delta^2}{2} > \LL(\beta_0, \varphi_0). 
\end{align*}
Since the inequality holds for all $\beta \in \cB$ such that $\norm{\beta-\beta_0}_2 > \delta$, we conclude that
\begin{align*}
    \inf \{\LL(\beta, \varphi_0)  \colon \norm{\beta-\beta_0}_2 > \delta\} > \LL(\beta_0, \varphi_0).
\end{align*}
Since $\delta > 0$ was arbitrary, the claim follows.

\subsubsection{Proof of \Cref{lm:ulln-2}}
    Recall that for any $\beta \in \cB$
  \begin{align*}
      \LL(\beta, \varphi_0) = \EE[g_{\beta, \varphi_0}(X_0, Y_0)],
      \quad
      \LL_n(\beta, \varphi_0) = \frac{1}{n_0}\sum_{i \in \cD_0} g_{\beta, \varphi_0}(X_{0,i}, Y_{0,i}).
  \end{align*}
    To show the result, we must establish that the class of functions $\{g_{\beta, \varphi_0} \colon \beta \in \cB\}$ is Glivenko--Cantelli. From~\cite{van2000asymptotic}, a set of sufficient conditions for being a Glivenko--Cantelli class is that (i) $\cB$ is compact, (ii) $\beta \mapsto g_{\beta, \varphi_0}(x, y)$ is continuous for every $(x, y)$, and (iii) $\beta \mapsto g_{\beta, \varphi_0}$ is dominated by an integrable function.
    By assumption, (i) holds.
    Moreover, by~\eqref{eq:g-func}, it follows that $\beta \mapsto g_{\beta, \varphi_0}$ is continuous for all $(x, y)$ and thus (ii) holds. We now show that (iii) holds. 
    Since $\cB$ is compact we have that $\sup_{\beta\in\cB} \norm{\beta}_2 = C_1 < \infty$.
    For fixed $\gamma > 0$, and all $(x, y)$, we have that
    \begin{align}
    \label{eq:dominates-i}
    \begin{split}
        g_{\beta, \varphi_0}(x, y) 
        \leq &\ \sup_{\beta \in \cB} |g_{\beta, \varphi_0}(x, y)|\\
        \leq &\ \sup_{\beta \in \cB} (y - \beta^\top x)^2
        +
        2\gamma  \norm{\PSMpop^{1/2}}_F^2 \left(\norm{\betastarS}_2^2 + \sup_{\beta \in \cB} \norm{\beta}_2^2\right)\\
        &+   \gamma \left(\sqrt{C - \norm{\betastarS}_2^2} 
        + \norm{\PSperpMpop}_F \sup_{\beta\in\cB}\norm{\beta}_2\right)^2\\
        \leq &\
        2y^2 + 2C_1^2 \norm{x}_2^2 + K \eqqcolon G(x, y),
    \end{split}
    \end{align} 
    where $K < \infty$ is a finite constant not depending on $(x, y)$.
    Furthermore, we have that
    \begin{align}\label{eq:dominates-ii}
        \EE[G(X_0, Y_0)] = 2 \EE[Y_0^2] + 2C_1^2\ \trace (\EE[X_0X_0^\top]) + K < \infty,
    \end{align}
    since $\EE[Y^2] < \infty$ and $\EE[X_0X_0^\top]$has bounded eigenvalues by assumption. From~\eqref{eq:dominates-i} and~\eqref{eq:dominates-ii}, it follows that (iii) holds.

\subsubsection{Proof of \Cref{lm:lipschitz}}
For fixed $\gamma > 0$, we have that
\begin{align}
    \frac{1}{\gamma} & \sup_{\beta \in \cB}|\LL_n(\beta, \hat\varphi) - \LL_n(\beta, \varphi_0)|
    \leq
    \sup_{\beta \in \cB}\left|
    \norm{\PSM^{1/2}(\hat\beta^{\cS} - \beta)}_2^2
    - \norm{\PSMpop^{1/2}(\betastarS - \beta)}_2^2
    \right| \label{eq:lip-step-1}
    \\
    & + 
    \sup_{\beta \in \cB}\left |\left(\sqrt{C - \norm{\hat{\beta}^{\cS}}_2^2}  + \norm{\PSperpM \beta}_2\right)^2 
    - \left(\sqrt{C - \norm{\betastarS}_2^2}  + \norm{\PSperpMpop \beta}_2\right)^2\right|.
    \label{eq:lip-step-2}
\end{align}
We first show that ~\eqref{eq:lip-step-1} converges in probability to 0. 
\begin{align}
    &\sup_{\beta \in \cB}\left|
    \norm{\PSM^{1/2}(\hat{\beta}^{\cS} - \beta)}_2^2
    - \norm{\PSMpop^{1/2}(\betastarS - \beta)}_2^2
    \right| \notag\\
    = &\ 
    \sup_{\beta \in \cB}
    \left|(\hat{\beta}^{\cS} - \beta)^\top \PSM(\hat{\beta}^{\cS} - \beta) 
    - (\betastarS - \beta)^\top \PSMpop (\betastarS - \beta) \right| \notag\\
    = &\
    \sup_{\beta \in \cB}\left|(\hat{\beta}^{\cS} - \beta)^\top \PSM (\hat{\beta}^{\cS} - \betastarS)
    + (\hat{\beta}^{\cS} - \betastarS)^\top \PSM(\betastarS - \beta) \right. \notag \\
    & \left. \quad\quad + (\betastarS - \beta)^\top (\PSM - \PSMpop) (\betastarS - \beta) \right| \notag\\
    \stackrel{\clubsuit}{\leq}&\ 
     \sup_{\beta \in \cB} \norm{\hat{\beta}^{\cS} - \beta}_2\ \norm{\PSM}_F\ \norm{\hat{\beta}^{\cS} - \betastarS}_2
     + \sup_{\beta \in \cB} \norm{\betastarS - \beta}_2\ \norm{\PSM}_F\ \norm{\hat{\beta}^{\cS} - \betastarS}_2 \notag \\
     & \left. \quad\quad +
     \sup_{\beta \in \cB} \norm{\betastarS - \beta}_2^2\ \norm{\PSM - \PSMpop}_F\right. \label{eq:bound-1-1},
\end{align}
where $\clubsuit$ follows from the Cauchy--Schwarz inequality and from the fact that $\norm{A}_2 \leq \norm{A}_F$.
For any $\delta_1, \delta_2 > 0$, define the event
\begin{align*}
    A_n \coloneqq \left\{\norm{\betaShat - \betastarS} \leq \delta_1, \norm{\PSM - \PSMpop}_F \leq \delta_2\right\},
\end{align*}
and note that $\prob(A_n) \to 1$ as $n \to \infty$ from Assumption~\ref{ass:consistency-nuisance}.
On the event $A_n$, it holds
\begin{align*}
    \sup_{\beta \in \cB} \norm{\hat{\beta}^{\cS} - \beta}_2 \leq \norm{\hat{\beta}^{\cS} - \betastarS}_2 + \sup_{\beta \in \cB}  \norm{\betastarS - \beta}_2 \leq \delta_1 + C_1,\\
    \norm{\PSM}_F \leq \norm{\PSM - \PSMpop}_F + \norm{\PSMpop}_F \leq \delta_2 +  C_2,   
\end{align*}
where $C_1 < \infty$ follows from the compactness of~$\cB$ and $C_2$ follows from the fact that  $\PSMpop$ has bounded eigenvalues.
Therefore, on the event $A_n$, we can upper bound~\eqref{eq:bound-1-1} by
\begin{align}
    (\delta_1 + C_1) (\delta_2 + C_2)\ \norm{\hat{\beta}^{\cS} - \betastarS}_2 + (\delta_1 + C_1)\ \norm{\PSM - \PSMpop}_F.\label{eq:bound-1-2} 
\end{align}
From Assumption~\ref{ass:consistency-nuisance}, \eqref{eq:bound-1-2} converges to 0 in probability, and therefore,~\eqref{eq:lip-step-1} converges to 0 in probability as well.

Now, we can upper bound~\eqref{eq:lip-step-2} as follows,
\begin{align}
    \sup_{\beta \in \cB} &\left |\left(\sqrt{C - \norm{\hat{\beta}^{\cS}}_2^2}  + \norm{\PSperpM \beta}_2\right)^2 
    - \left(\sqrt{C - \norm{\betastarS}_2^2}  + \norm{\PSperpMpop \beta}_2\right)^2\right| \notag\\
    = &\ \sup_{\beta \in \cB} \left|
    C - \norm{\hat{\beta}^{\cS}}_2^2
    + \norm{\PSperpM \beta}_2^2 
    + 2 \sqrt{C - \norm{\hat{\beta}^{\cS}}_2^2}\ \norm{\PSperpM \beta}_2\right.\notag\\
    &\phantom{\sup_{\beta \in \cB}\ }\left.
    - C + \norm{\betastarS}_2^2 
    - \norm{\PSperpMpop \beta}_2^2
    - 2 \sqrt{C - \norm{\betastarS}_2^2}\  \norm{\PSperpMpop \beta}_2
    \right| \notag\\
    \leq &\ \left|\norm{\hat{\beta}^{\cS}}_2^2 - \norm{\betastarS}_2^2\right|
    + \sup_{\beta \in \cB} \left|\beta^\top (\PSperpM - \PSperpMpop) \beta\right|\notag\\ 
    &+2 \sup_{\beta \in \cB}\left|\sqrt{C - \norm{\hat{\beta}^{\cS}}_2^2}\ \norm{\PSperpM \beta}_2-\sqrt{C - \norm{\betastarS}_2^2}\ \norm{\PSperpMpop \beta}_2\right|\notag\\
    = &\ (I) + (II) + (III). \notag 
\end{align}
By Assumption~\ref{ass:consistency-nuisance}, $(I)$ converges in probability to zero. 
Regarding $(II)$, we have
\begin{align*}
    \sup_{\beta \in \cB} \left|\beta^\top (\PSperpM - \PSperpMpop) \beta\right| 
    \leq 
     \sup_{\beta \in \cB} 
     \norm{\beta}_2^2\
     \norm{\PSperpM - \PSperpMpop}_F
     \stackrel{\prob}{\to}0,
\end{align*}
where the inequality follows from Cauchy--Schwarz and that $\norm{A}_2 \leq \norm{A}_F$, and the convergence in probability follows from Assumption~\ref{ass:consistency-nuisance} along with the compactness of~$\cB$.
It remains to upper bound $(III)$. We have that
\begin{align}
    \frac{(III)}{2} 
 \leq 
 &\
 \sup_{\beta \in \cB}
 \left|
 \sqrt{C - \norm{\hat{\beta}^{\cS}}_2^2}\ \norm{\PSperpM \beta}_2
  - \sqrt{C - \norm{\betastarS}_2^2}\ \norm{\PSperpM \beta}_2
 \right| \notag\\
 &+
  \sup_{\beta \in \cB}
 \left|
 \sqrt{C - \norm{\betastarS}_2^2}\ \norm{\PSperpM \beta}_2
  - \sqrt{C - \norm{\betastarS}_2^2}\ \norm{\PSperpMpop \beta}_2
 \right| \notag\\
 \leq &\
 \left(\sup_{\beta \in \cB} \norm{\beta}_2\ \norm{\PSperpM}_F\right)\
 \left|
 \sqrt{C - \norm{\hat{\beta}^{\cS}}_2^2} 
 - 
 \sqrt{C - \norm{\betastarS}_2^2}
 \right|
 \notag\\
 &+\sup_{\beta \in \cB}
 \left|
 \sqrt{\beta^\top \PSperpM  \beta}
 -
 \sqrt{\beta^\top \PSperpMpop \beta}
 \right|
 \left(
 \sqrt{C - \norm{\betastarS}_2^2}\right)\notag\\
 \leq &\
 C_3 
 \left| \norm{\betastarS}_2^2 - \norm{\hat{\beta}^{\cS}}_2^2
 \right|^{1/2}
 + 
 \sqrt{C} 
 \sup_{\beta \in \cB} 
  \left|
 \beta^\top ( \PSperpM- \PSperpMpop) \beta
  \right|^{1/2}
  \label{eq:bound-sqrt-1}\\
  \leq &\
   C_3 
 \left| \norm{\betastarS}_2^2 - \norm{\hat{\beta}^{\cS}}_2^2
 \right|^{1/2}
 + 
 \sqrt{C} 
 \left(
    \sup_{\beta \in \cB} \norm{\beta}_2^2\ \norm{ \PSperpM- \PSperpMpop}_F
 \right)^{1/2} \stackrel{\prob}{\to} 0.\label{eq:bound-sqrt-2}
\end{align}
The inequality in~\eqref{eq:bound-sqrt-1} follows from the compactness of $\cB$, the fact that $\PSperpM$ has bounded eigenvalues, and that $|\sqrt{x} - \sqrt{y}| \leq |x - y|^{1/2}$ for all $x, y \geq 0$.
The inequality in~\eqref{eq:bound-sqrt-2} follows from Cauchy--Schwarz and that $\norm{A}_2 \leq \norm{A}_F$.
The convergence in probability follows form Assumption~\ref{ass:consistency-nuisance} and the compactness of~$\cB$.

\section{Details on finite-sample experiments}\label{sec:apx-experiments}

In this section, we provide more details of the data generation for our synthetic finite-sample experiments as well as data processing for the real-world data experiments.
\subsection{Synthetic experiments}\label{sec:apx-synthetic-exps}

For the synthetic experiments, we generate a random SCM which satisfies our assumptions. For $d = 15$, we randomly sample the joint covariance $\Sigmastar$ of $(\eta,\xi)$, fixing its total variance and the eigenvalues. We consider 7 environments including the reference environment, and for each environment except the reference, we randomly generate mean shifts $\mue$ of fixed norm $1$. Since we have $6$ non-zero random Gaussian mean shifts, it holds a.s. that $\dim \cS = 6$. We then randomly generate an "initial guess" for $\betastar \in \R^d$ of fixed norm $C = 10$. Now, with respect to the space $\cS$ of the identifiable directions induced by the mean shifts, we choose the most "adversarial" causal parameter $\betaadv$ which is equal to $\betastar$ on $\cS$, but on $\cSperp$ has the opposite direction of the noise OLS estimator ${\noisecovxxstar}^{-1} \noisecovxystar$. We ensure that $\| \betaadv \|_2 = C$. Note that under the observed shifts, $\betastar$ and $\betaadv$ are observationally equivalent. We complete $\betaadv$ to the set $\thetaadv$ of observationally equivalent model parameters and generate the multi-environment training data according to $\thetaadv$ and the collection of mean shifts. 

For \cref{fig:synthetic-experiments} (left), we define the test shift upper bound as $\Manchor = \gamma \frac{1}{7} \sum_{e} \mu_e \mu_e^\top$. We vary $\gamma$ from $0$ to $10$, and for each $\gamma$, we compute the oracle anchor regression estimator by minimizing the discrete anchor regression loss with the correct $\gamma$. Additionally, we compute the pooled OLS estimator and the worst-case robust predictor $\betarobpi$ as described in \cref{sec:apx-empirical-estimation}. Finally, we generate test data with a Gaussian additive shift $\Atest \sim \cN(0, \Manchor)$. We evaluate the loss of $\betaOLS$, $\betaa$ and $\betarobpi$ on this test environment and include the population lower bound. 

For \cref{fig:synthetic-experiments} (right), we define the test shift upper bound as $\Mnew = \gamma \frac{1}{7} \sum_{e} \mu_e \mu_e^\top + \gammaprime R R^\top$, where $R$ is a 2-dimensional subspace of the space $\cSperp$. We fix the magnitude $\gamma$ of the ''seen'' test shift directions at $\gamma = 40$ and set vary $\gammaprime$ from $0$ to $2$ to showcase the effect of small unseen shifts compared to large identified shifts. We compute the oracle anchor regression estimator by minimizing the discrete anchor regression loss. Additionally, we compute the pooled OLS estimator and the worst-case robust predictor $\betarobpi$ as described in \cref{sec:apx-empirical-estimation}, for which we use the oracle $\gammaprime$, given $\Manchor$ and empirical estimates of the spaces $\cS$, $\cSperp$, $R$. \\
Finally, we generate test data with a Gaussian additive shift $\Atest \sim \cN(0, \Mnew)$. We evaluate the loss of $\betaOLS$, $\betaa$ and $\betarobpi$ on this test environment, plot the resulting test losses for different estimators and include the population lower bound. 

\subsection{Real-world data experiments}\label{sec:apx-real-world}

\begin{figure}[h]
    \centering
    \begin{subfigure}[b]{0.3\textwidth}
        \centering
        \includegraphics[width=\textwidth]{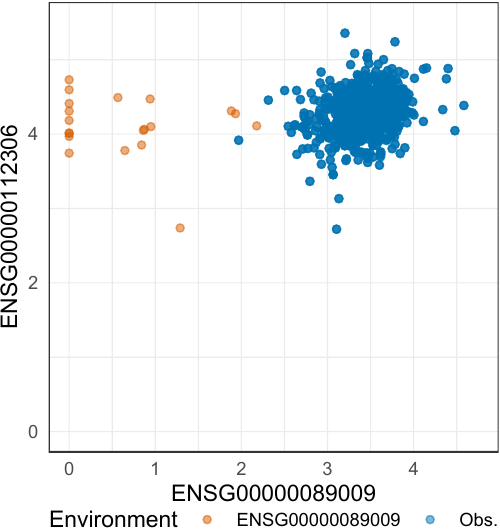}
        \caption{}
        \label{fig:fig1}
    \end{subfigure}
    \hfill
    \begin{subfigure}[b]{0.3\textwidth}
        \centering
        \includegraphics[width=\textwidth]{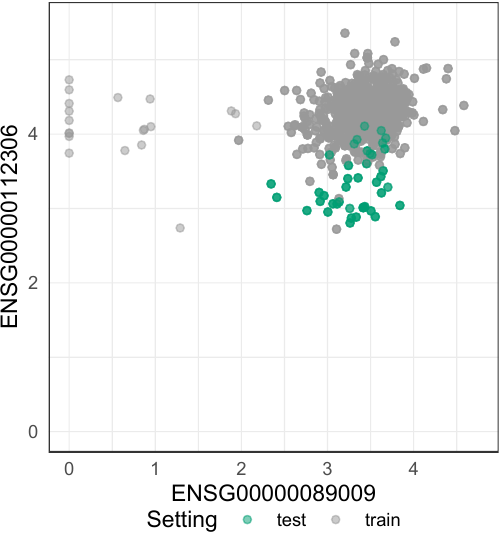}
        \caption{}
        \label{fig:fig2}
    \end{subfigure}
    \hfill
    \begin{subfigure}[b]{0.3\textwidth}
        \centering
        \includegraphics[width=\textwidth]{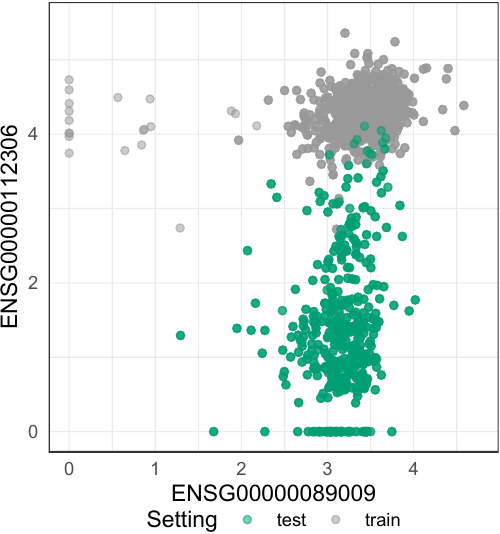}
        \caption{}
        \label{fig:fig3}
    \end{subfigure}
    \caption{The figures illustrate the structure of the (a) training-time shifts and (b-c) test-time shifts for different perturbation strengths on the example of two covariates. Panel (a) shows the training data containing two environments--observational (blue) and shifted (orange) corresponding to the knockout of the gene ENSG00000089009. 
    Panels~(b) and~(c) show the training data in grey and test data from a previously unseen environment (green). 
    Panel~(b) depicts the top $10\%$ test data points closest to the training support (perturbation strength = $0.1$).
    Panel~(c) illustrates the full test data (perturbation strength = 1.0).
    }
    \label{fig:genes}
\end{figure}

We consider the K562 dataset from \cite{replogle2022mapping} and perform the preprocessing as done in \cite{chevalley2022causalbench}.
The resulting dataset consists of $n = 162,751$ single-cell observations over $d = 622$ genes collected from observational and several interventional environments. 
The interventional environments arise by knocking down a single gene at a time using the CRISPR interference method \citep{qi2013repurposing}. Following \citep{schultheiss2024assessing}, we select only always-active genes in the observational setting, resulting in a smaller dataset of 28 genes. For each gene $j = 1, \ldots, 28$, we set $Y:= X_j$ as the target variable and select the three genes $X_{k_1}, \ldots, X_{k_3}$ most strongly correlated with $Y$ (using Lasso), resulting in a prediction problem over $Y, X_{k_1}, \ldots, X_{k_3}$.
Given this prediction problem, we construct the training and test datasets as follows. Let $\Iobs$ denote the 10,691 observations collected from the observational environment, and let $\Iint_{i}$ denote the observations collected from the interventional environment where the gene $k_i$ was knocked down. We will denote by $\Iint_{i,s}$ the $s \times 100$ percent of datapoints in $\Iint_{i}$ that are closest to the mean of gene $k_i$ in the observational environment $\Iobs$. 
For example, $\Iint_{i,0.1}$ consists of the 10\% of datapoints in $\Iint_{i}$ closest to the observational mean of gene $k_i$.
Thus, 
the parameter $s \in [0,1]$ acts as a proxy for the \emph{strength} of the shift. 
Denote by $\Iint_{i, s}^*$ a random sample of $\Iint_{i, s}$ of a certain size.
For each $i \in \{1, 2, 3\}$, we fit the methods on the training data  $\Dtrain_i \coloneqq \Iobs \cup \Iint_{i, 1}^*$, with $|\Iint_{i, 1}^*| = 20$. \cref{fig:genes}(a) illustrates an example of training data  $\Dtrain_i$.
Having fitted the methods on $\Dtrain_i$, we evaluate them on test datasets constructed as follows.
For each 
shift strength $s \in \{0.1, \dots, 0.9\}$ and proportion $\pi \in \{0, .33, .67, 1\}$, define the test dataset $\mathcal{D}_{\pi, s}^{\mathrm{test}}$ consisting of $\pi$ observations from $\cup_{\ell\neq i}\Iint_{\ell, s}$ and $1-\pi$ (out-of-training) observations from $\Iint_{i, s}$.
An example of a test dataset for different shift strengths $s$ and previously unseen directions (i.e., $\pi = 1$) is shown in \cref{fig:genes}(b-c).  
We compare our method Worst-case Rob., defined as the minimizer of the empirical worst-case robust risk \eqref{eqn:rob-loss-ell-sample}, with anchor regression \citep{rothenhausler2021anchor}, invariant causal prediction (ICP) \citep{peters2016causal},  Distributional Robustness via Invariant Gradients (DRIG) \citep{shen2023causalityoriented}, and OLS (corresponding to vanilla ERM).
We use the following parameters for Worst-case Rob.: $\gamma = 50$, $\Cker = 1.0$, and $M = \Id$. For anchor regression and DRIG, we select $\gamma = 50$. For ICP, we set the significance level for the invariance tests to $\alpha = 0.05$.

These numerical experiments are computationally light and can be run in $\approx 5$ minutes on a personal laptop.\footnote{We use a 2020 13-inch MacBook Pro with a 1.4 GHz Quad-Core Intel Core i5 processor, 8 GB of RAM, and Intel Iris Plus Graphics 645 with 1536 MB of graphics memory.}

\section{Proofs}
\label{sec:apx-proofs}
\subsection{Proof of \cref{prop:invariant-set}}\label{sec:apx-proof-invariant-set}
For every environment $e \in \Ecaltrain$, we observe the first moments $\EE(X_e)$ and $\EE(Y_e)$,
and second moments $\EE(X_eX_e^\top)$, $\EE(Y_e^2)$ and $\EE(X_eY_e)$.
Since it holds by assumption that $\mu_0 = 0$ and $\Sigma_0 = 0$, we have that  $\EE(X_0X_0^\top) = \noisecovxxstar$, and so we can identify $\noisecovxxstar$ uniquely. Furthermore, it holds that
\begin{align}
    \EE(X_0Y_0) &= \noisecovxxstar \betastar + \noisecovxystar, \label{eqn:ident-1}\\
    \EE(X_eY_e) &= ( \Sigma_e + \mu_e \mu_e^\top + \noisecovxxstar) \betastar + \noisecovxystar.
    \label{eqn:ident-2}
\end{align}
By taking the difference between \cref{eqn:ident-2} and \cref{eqn:ident-1},
we can identify $(\Sigma_e + \mu_e \mu_e^\top) \betastar$.
Thus, 
the parameter $\betastar$ is identifiable on the subspace $\cS$ defined in \cref{eqn:def-S} and is not identifiable on its orthogonal complement $\cSperp$.
Thus, for any vector $\alpha \in \cSperp$ , the vector $\beta = \betastar + \alpha$ is consistent with the data-generating process. It remains to compute the covariance parameters induced by an arbitrary $\tilde\beta \coloneqq \betastar + \alpha$, for $\alpha \in \cSperp$. For every environment $e \in \Ecaltrain$,  the second mixed moment between $X_e$ and $Y_e$ has to satisfy the following equality
\begin{align*}
    \EE(X_eY_e) = (\Sigma_e + \mu_e\mu_e^\top + \noisecovxxstar)\betastar + \noisecovxystar = (\Sigma_e + \mu_e\mu_e^\top + \noisecovxxstar) \tilde{\beta}+ \tilde{\Sigma}_{\eta, \xi},
\end{align*}
from which it follows that $\tilde{\Sigma}_{\eta, \xi}
 \coloneqq \noisecovxystar - \noisecovxxstar \alpha$. By computing $\EE(Y_e^2)$ and inserting $\tilde{\beta} = \betastar + \alpha$ and $\tilde{\Sigma}_{\eta, \xi}$, we similarly obtain 
\begin{align*}
    \tilde{\sigma}_{\xi}^{2} \coloneqq \noisecovyystar - 2 \alpha^\top \noisecovxystar + \alpha^\top \noisecovxxstar \alpha. 
\end{align*}
Thus, we obtain the following set of observationally equivalent model parameters consistent with $\probtrainarg{\thetastar}$:
\begin{align*}
    \Invset = \{ \betastar + \alpha, \noisecovxxstar, \noisecovxystar - \noisecovxxstar \alpha, \noisecovyystar - 2 \alpha^\top \noisecovxystar + \alpha^\top \noisecovxxstar \alpha \colon \alpha \in \cSperp \}. 
\end{align*}
Since the \idset is identifiable from the training distribution, but model parameters $\betastar$, $\noisecovxystar$, $\noisecovyystar$ are not, it is helpful to re-express the \idset through identifiable quantities. For this, we note that the "identifiable linear predictor" $\betastarS = \betastar - \betastarperp$ induces an observationally equivalent model given by 
\begin{align*}
    \thetastarS := (\betaS, \noisecovxxS, \noisecovxyS, \noisecovyyS) = (\betastarS, \noisecovxxstar, \noisecovxystar + \noisecovxxstar \betastarperp, \noisecovyystar + 2 \langle \noisecovxystar, \betastarperp\rangle + \langle \betastarperp, \noisecovxxstar \betastarperp\rangle).
\end{align*}
From this reparameterization, we infer the final form of the \idset:
\begin{align*}
   \Invset = \{ \betastarS + \alpha, \noisecovxx', \noisecovxyS - \noisecovxx' \alpha, \noisecovyyS - 2 \alpha^\top \noisecovxyS + \alpha^\top \noisecovxx' \alpha \colon \alpha \in \cSperp \}  \ni \thetastar 
\end{align*}
Therefore, \cref{eqn:def-invariant-set} follows.
To find the robust predictor $\betarob$, we write down the robust loss with respect to $\Mtest$ and any $\theta_\alpha$ from the \idset:
\begin{align*}
    \Lossrob(\beta;\theta_\alpha, \Mtest) &= (\betastarS + \alpha - \beta)^\top (\Mtest + \noisecovxxstar) (\betastarS + \alpha - \beta) \\ &+ 2 (\betastarS + \alpha - \beta)^\top (\noisecovxystar - \noisecovxxstar \alpha) + \noisecovyyS - 2 \alpha^\top \noisecovxyS + \alpha^\top \noisecovxxstar \alpha.
\end{align*}
inserting $\alpha \in \cSperp$ and rearranging, \cref{eqn:def-rob-pred-identif} follows.

\subsection{Proof of \cref{thm:pi-loss-lower-bound}}\label{sec:apx-proof-of-main-prop}

We structure the proof as follows: first, we quantify the non-identifiability of the robust risk by explicitly computing its supremum over the \idset of the model parameters (referred to as the \idRR). Second, we derive a lower bound for the \idRRs by considering two cases depending on how a predictor $\betabar$ interacts with the possible test shifts $\Mtest$. 
\paragraph{Computation of the \idRR.} For any model-generating parameter $\theta = (\beta, \Sigma)$ it holds that the robust risk of the model \cref{eqn:SCM} under test shifts $\Mtest \succeq 0$ is given by 
\begin{align*}
    \Lossrob(\betabar;\theta, \Mtest) =  (\beta - \betabar)^\top(\Mtest + \noisecovxxstar)(\beta - \betabar) + 2(\beta - \betabar)^\top \noisecovxy + \noisecovyy. 
\end{align*}
We recall that the \idset of model parameters after observing the multi-environment training data \cref{eqn:SCM} is given by 
\begin{align}\label{eqn:apx-def-invariant-set}
     \Invset = \{ \betastarS + \alpha, \noisecovxxstar, \noisecovxyS - \noisecovxxstar \alpha, \noisecovyyS - 2 \alpha^\top \noisecovxyS + \alpha^\top \noisecovxx \alpha: \alpha \in \cSperp \},
\end{align}
where $\cS$ is the span of identified directions defined in \cref{eqn:def-S}. 
Moreover, we recall that by Assumption~\ref{as:bounded-betastar}, for any causal parameter $\beta$ it should hold that $\| \beta \|_2 = \| \betastarS + \alpha \|_2 \leq C$, which translates into the following constraint for the parameter $\alpha$:
\begin{align*}
    \| \alpha \|_2 \leq \sqrt{C^2 - \| \betastarS \|_2^2} =: \Cker. 
\end{align*}
Inserting \cref{eqn:apx-def-invariant-set} in \cref{eqn:PI-robust-loss}, we obtain
\begin{align*}
    \Lossrobpi(\betabar; \Invset, \Mtest) = \supalpha \Lossrob(\betabar; \theta_{\alpha}, \Mtest),
\end{align*}
where $\theta_\alpha$ is a short notation for $(\betastarS + \alpha, \noisecovxxstar, \noisecovxyS - \noisecovxxstar \alpha, \noisecovyyS - 2 \alpha^\top \noisecovxyS + \alpha^\top \noisecovxxstar \alpha)$. We now compute the supremum explicitly in case $\Mtest$ has the form $\Mtest = \gamma \Mseen + \gammaprime R R^\top$, where $\Mseen$ is a PSD matrix with $\range(M) \subseteq \cS$ and $R$ is a semi-orthogonal matrix with $\range(R) \subseteq \cSperp$. For any $\alpha \in \cSperp$, we write down the robust loss as
\begin{align*}
    \Lossrob(\betabar; \theta_\alpha, \Mtest) &= (\betastarS - \betabar)^\top (\Mtest + \noisecovxxstar) (\betastarS - \betabar) + 2 (\betastarS - \betabar)^\top \noisecovxyS + \noisecovyyS \\
    &+ \alpha^\top \Mtest \alpha + 2 \alpha^\top \Mtest(\betastarS -  \betabar ) \\
    &= \Lossrob(\betabar; \thetastarS, \Mtest) + \alpha^\top \Mtest \alpha + 2 \alpha^\top \Mtest(\betastarS -  \betabar ). 
\end{align*}
The first term is the robust risk of $\betabar$ under test shift $\Mtest$ and the identified model-generating parameter $\thetastarS$, thus it does not depend on $\alpha$. 
By the structure of $\Mtest$, we obtain that 
\begin{align*}
    f(\alpha) := \alpha^\top \Mtest \alpha + 2 \alpha^\top \Mtest(\betastarS -  \betabar )  = \gammaprime \alpha^\top R R^\top \alpha - \gammaprime \alpha^\top R R^\top \betabar. 
\end{align*}
If $\gammaprime = 0$, i.e., the test shifts consist only of the identified directions, we have $f(\alpha) = 0$, independently of $\alpha$, and thus 
\begin{align*}
     \Lossrobpi(\betabar; \Invset, \Mtest) = \Lossrob(\betabar; \thetastarS, \Mtest).
\end{align*}
This implies the first statement of the theorem. 
\par
We now consider the case where $R \neq 0$, i.e., $R R^\top$ is a non-degenerate projection.
Our goal is to maximize $f(\alpha)$ subject to constraints $\alpha \in \cSperp$, $\| \alpha \|_2 \leq \Cker$. Let $\Rtilde$ be an orthonormal extension of $R$ such that $\range\ (R | \Rtilde) = \cSperp$. Then, we can parameterize $\alpha \in \cSperp$ as $\alpha = (R | \Rtilde) (\frac{w} {\wtilde})$ and the corresponding Lagrangian reads
\begin{align*}
    \mathcal{L}(\alpha, \lambda) &= \gammaprime \alpha^\top R R^\top \alpha - \gammaprime \alpha^\top R R^\top \betabar + \lambda(\Cker^2 - \| \alpha \|_2^2) \\ &= \gammaprime \| w \|_2^2 -\gammaprime w^\top R^\top \betabar + \lambda(\Cker^2 - \| (w, \wtilde) \|_2^2). 
\end{align*}
Differentiating with respect to $w, \wtilde$ yields
\begin{align*}
    w &= \frac{\gammaprime}{\gammaprime - \lambda} R^\top \betabar; \\
    \wtilde &= 0. 
\end{align*}
After differentiating w.r.t. $\lambda$, we obtain
$\frac{\gammaprime}{\gammaprime - \lambda} = \pm \frac{\Cker}{\| R^\top \betabar \|_2}$. By inserting in the objective function and comparing, we obtain the \textbf{value of the \idRR}: 
\begin{align}\label{eqn:proofs-id-robust-risk}
    \Lossrobpi(\betabar; \Invset, \Mtest) &= \gammaprime \Cker^2 + 2 \gammaprime \| R^\top \betabar \|_2 + \Lossrob(\betabar; \thetastarS, \Mtest) \\
    &= \gammaprime \Cker^2 + 2 \gammaprime \| R^\top \betabar \|_2 + \betabar^\top R R^\top \betabar + \gamma (\betastarS - \beta)^\top \Mseen (\betastarS - \beta) + \Loss_0 (\betabar,\thetastarS).
\end{align}
Putting together the two cases and simplifying, we obtain
\begin{align}\label{eqn:detailed-id-robust-risk}
\begin{split}
    \Lossrobpi(\betabar; \Invset, \Mtest) &= \gammaprime(\Cker + \| R^\top \betabar \|_2)^2 + \Lossrob(\betabar; \thetastarS, \gamma \Mseen) \\  &= \gammaprime  (\Cker + \| R^\top \betabar \|_2)^2 + \gamma (\betastarS - \betabar)^\top \Mseen (\betastarS-\betabar) + \Loss_0 (\betabar,\thetastarS), 
\end{split}
\end{align}
where $\Lossrob(\betabar; \thetastarS, \gamma \Mseen)$ is the robust risk of the estimator $\betabar$ w.r.t. the "identified" test shift $\gamma M$ and the identified model parameter $\thetastarS$, whereas $\Loss_0 (\betabar,\thetastarS)$ is the risk of $\betabar$ on the reference environment $e = 0$. 
\paragraph{Derivation of the lower bound for the \idRR.} Now that we have explicitly computed the \idRR, we devote ourselves to the computation of the lower bound for its best possible value
\begin{align*}
    \inf_{\betabar \in \R^d} \Lossrobpi(\betabar; \Invset, \Mtest). 
\end{align*}
In this part, we will only consider the case $R \neq 0$, since the case $R = 0$ corresponds to the (discrete) anchor regression-like setting, where both the robust risk and its minimizer are uniquely identifiable, and computable from training data. We will distinguish between two cases.
\paragraph{Case 1: $\| R^\top \betabar \|_2 = 0$.}  In this case, $\betabar$ is fully located in the orthogonal complement of $R$, which consists of $\cS$ and $\Rtilde$ (the orthogonal complement or $R$ in $\cSperp$). We will denote (the basis of) this subspace by $\Stot = \cS \oplus \Rtilde$. Thus, $\Stot$ is the "total" stable subspace consisting of identified directions in $\cS$ and non-identified, but unperturbed directions $\Rtilde$. We will parameterize $\betabar$ as $\betabar = \Stot w$. Thus, we are looking to solve the optimization problem 
\begin{align*}
   \betarobpi =  \argmin_{w} \, (\betastarS - \Stot w)^\top (\gamma \Mseen^\top + \noisecovxxstar) (\betastarS -  \Stot w) + 2 (\betastarS -  \Stot w)^\top \noisecovxyS + \noisecovyyS.
\end{align*}
Setting the gradient to zero yields the \emph{asymptotic worst-case robust estimator} 
\begin{equation}\label{eq:apx-pi-robust-formula}
\begin{aligned}
        \betarobpi &= \betastarS + \Stot [ \Stot^\top (\gamma \Mseen^\top + \noisecovxx) \Stot ]^{-1} \Stot^\top \noisecovxyS,
\end{aligned}
\end{equation}
which corresponds to the loss value of 
\begin{align*}
    \Lossrobpi(\betarobpi; \Invset, \Mtest) = \gammaprime \Cker^2 +  \noisecovyyS - 2 {\noisecovxyS}^\top \Stot [ \Stot^\top (\gamma \Mseen^\top + \noisecovxx) \Stot ]^{-1} \Stot^\top \noisecovxyS.
\end{align*}
As we observe, this quantity grows linearly in $\gammaprime$. However, as $\gamma \to \infty$, the quantity \emph{saturates} and is upper-bounded by $\noisecovyyS$.
\paragraph{Case 2:$\| R^\top \betabar \|_2 \neq 0$.} Since for $\| R^\top \betabar \|_2 \neq 0$, the objective function is differentiable, we compute its gradient to be
\begin{align*}
    \nabla  \Lossrobpi(\beta; \Invset, \Mtest) &= 2 \gammaprime R R^\top \beta / \| R R^\top \beta \| + 2 \gammaprime R R^\top \beta + \nabla \Lossrob(\beta; \thetastarS, \gamma \Mseen) \\ 
    &= 2 \gammaprime R R^\top \beta / \| R R^\top \beta \| + 2 \gammaprime R R^\top \beta  + 2(\noisecovxxstar + \gamma \Mseen) (\beta - \betastarS) - 2 \noisecovxyS. 
\end{align*}
This equation is, in general, not solvable w.r.t. $\beta$ in closed form. Instead, we provide the limit of the optimal value of the function when the strength of the unseen shifts is small, i.e. $\gammaprime \to 0$. We know that for $\gammaprime = 0$, the minimizer of the worst-case robust risk is given by the anchor estimator
\begin{align*}
    \betaa = \betastarS + (\noisecovxxstar + \gamma \Mseen)^{-1} \noisecovxyS. 
\end{align*}
Instead, we lower bound the non-differentiable term $2 \gammaprime \Cker \| R^\top \beta \|$ by the scalar product $2 \gammaprime \Cker \scalar{R^\top \beta}{R^\top \betaa}/ \| \betaa \|$ and expect it to be tight for small $\gammaprime$. After inserting this lower bound in \cref{eqn:proofs-id-robust-risk} we obtain the minimizer of the lower bound of form
\begin{align*}
    \beta_{LB} = \betastarS + (\noisecovxxstar + \gamma M + \gammaprime R R^\top)^{-1}(\noisecovxyS - \gammaprime \Cker R R^\top (\noisecovxxstar + \gamma M)^{-1} \noisecovxyS).
\end{align*}
We can now lower bound $\| R R^\top \beta_{LB} \|$ as 
\begin{equation}\label{eqn:small-gammaprime-lower-bound}
    \| R R^\top \beta_{LB} \| \geq \| R R^\top (\noisecovxxstar + \gamma M)^{-1} \noisecovxystar \| - \gammaprime \cdot \text{const}.
\end{equation}
Thus, the $\gammaprime$-rate of the \idRRs of $\beta_{LB}$ is at least $\gammaprime (\Cker + \| R R^\top (\noisecovxxstar + \gamma M)^{-1} \noisecovxystar \|)^2 + \mathcal{O}(\gammaprime^2)$,
from which the claim for small $\gammaprime$ follows. For \cref{sec:comp-with-finite-robustness-methods}, the lower bound directly implies optimality of the worst-case robust risk of the anchor estimator when the strength of the unseen shifts $\gammaprime$ is small. Additionally. if $\gamma = 0$, i.e. only unseen test shifts occur, we conclude that the OLS and anchor estimators have the same rates. 
\paragraph{Lower bound $\gammath$ for $\gammaprime$.}
Finally, we want to derive a lower bound on the shift strength  $\gammaprime$ such that for all $\gammaprime \geq \gammath$ Case 1 of our proof is valid, i.e. it holds that $\betarobpi$ is given by the closed form "abstaining" estimator \eqref{eq:apx-pi-robust-formula}. For this, we find $\gammath$ such that for all $\gammaprime \geq \gammath$ zero is contained in the subdifferential of$\Lossrobpi(\betarobpi;\Invset,\Mtest)$ at $\betarobpi$. Then the KKT conditions are met, and $\betarobpi$ is the unique minimizer of the worst-case robust risk due to strong convexity of the objective. We compute the subdifferential to be
\begin{align*}
    S = \gammaprime \Cker \{ R R^\top \beta: \| \beta \|_2 \leq 1 \} + \nabla \Lossrob(\betarobpi; \thetastarS, \gamma M).
\end{align*}
Since $\betarobpi$ is the minimizer of $ \Lossrob(\beta; \thetastarS, \gamma \Mseen)$ under the constraint $R^\top \beta = 0$, the gradient is zero in $R^\perp$ and it remains to show that 
\begin{align*}
    \| R R^\top \nabla \Lossrob(\betarobpi; \thetastarS, \gamma \Mseen) \| \leq \gammaprime \Cker,
\end{align*}
or 
\begin{align*}
    \gammaprime \geq \| R R^\top \nabla \Lossrob(\betarobpi; \thetastarS, \gamma \Mseen) \| / \Cker. 
\end{align*}
Via an upper bound on the projected gradient, we derive the stricter condition
\begin{align*}
    \gammaprime \geq \frac{\| R R^\top \noisecovxyS\| (1 + \kappa(\noisecovxxstar)) }{\Cker},
\end{align*}
where $\kappa(\noisecovxxstar)$ is the condition number of the covariance matrix. 

\subsection{Proof of \cref{cor:estimators}}\label{sec:apx-proof-of-corollary}
To obtain a new formulation for the \idRR, we start with \eqref{eqn:detailed-id-robust-risk} and expand 
\begin{align}
\begin{split}
    \Lossrobpi(\betabar; \Invset, \Mtest) &= \gammaprime  (\Cker + \| R^\top \betabar \|_2)^2 + \gamma (\betastarS - \betabar)^\top \Manchor (\betastarS-\betabar) + \Loss_0 (\betabar,\thetastarS) \\
    &= \gammaprime  (\Cker + \| R^\top \betabar \|_2)^2 + \gamma (\betastarS - \betabar)^\top \Manchor (\betastarS-\betabar) \\&+ (\betastarS - \betabar)^\top \noisecovxxstar (\betastarS-\betabar) + 2 (\betastarS - \beta)\noisecovxyS + \noisecovyyS \\ &= \gammaprime  (\Cker + \| R^\top \betabar \|_2)^2 + (\gamma - 1)(\betastarS - \betabar)^\top \Manchor (\betastarS-\betabar) + \Loss(\beta,\ptrain) \\
    &= \Lossrob(\beta, \thetastarS, \gamma \Manchor) + \gammaprime  (\Cker + \| R^\top \betabar \|_2)^2,
\end{split}
\end{align}
where we have used that the pooled second moment of $X$ equals to $\noisecovxxstar + \sum_e w_e (\mu_e \mu_e^\top) = \noisecovxxstar + \gamma \Manchor - (\gamma - 1) \Manchor$.
This reformulation shows that the \idRRs is equal to the anchor population loss (cf. \cite{rothenhausler2021anchor}) with an additional non-identifiability penalty term $\gammaprime  (\Cker + \| R^\top \betabar \|_2)^2$. 

We now want to evaluate the rates of the anchor and OLS estimators in terms of the magnitude $\gammaprime$ of unseen shift directions. We observe that only the non-identifiability term depends on $\gammaprime$, whereas the second term only depends on $\gamma$. First, we compute the closed-form anchor regression estimator, which reads 
\begin{equation}
    \betaa = \argmin_{\beta \in \R^d} \Lossrob(\beta, \thetastarS, \gamma \Manchor) = \betastarS + (\noisecovxxstar + \gamma \Manchor)^{-1} \noisecovxyS.
\end{equation}
Since $\betaOLS$ equals to the anchor estimator with $\gamma = 1$, we obtain 
\begin{equation*}
    \betaOLS = \betastarS + (\noisecovxxstar + \Manchor)^{-1} \noisecovxyS.
\end{equation*}
The claim of the corollary now follows by computing $\| R R^\top \betaa \|$ and $\| R R^\top \betaOLS \|$ and observing that the rest of the terms is constant in $\gammaprime$. Additionally, we observe that $\betaa$ is the minimizer of $\Lossrob(\beta, \thetastarS, \gamma \Manchor)$, and it is known (cf. e.g. \cite{rothenhausler2021anchor}) that $\Lossrob(\betaa, \thetastarS, \gamma \Manchor)$ is asymptotically constant in $\gamma$ and upper bounded by $\noisecovyyS$. On the other hand, the term $\Lossrob(\betaOLS, \thetastarS, \gamma \Manchor)$ is linear in $\gamma$. In total, we obtain 
\begin{equation*}
    \begin{aligned}
        \Lossrobpi(\betaa; \Invset, \Mtest) &= (\Cker + \| R R^\top \betaa \| )^2\gammaprime + c_1(\gamma); \\ 
        \Lossrobpi(\betaOLS; \Invset,\Mtest) &= (\Cker + \| R R^\top \betaOLS \| )^2\gammaprime + c_2(\gamma),
    \end{aligned}
\end{equation*}
where $c_1(\gamma) \leq \noisecovyyS$ and $c_2(\gamma) = \Omega(\gamma)$.
Comparing to the lower bound \eqref{eqn:small-gammaprime-lower-bound} for the minimax quantity for the case of $\gammaprime \to 0$, we observe that the anchor estimator is optimal (achieves the minimax rate) in the limit $\gammaprime \to 0$. Additionally, if $\gamma = 0$ (only new shifts occur during test time), anchor and OLS have identical rates in $\gammaprime$ and, in particular, OLS (corresponding to vanilla empirical risk minimization) is minimax-optimal in the limit of small unseen shifts. 
In the proof of \cref{thm:pi-loss-lower-bound} in \cref{sec:apx-proof-of-main-prop}, we show that for $\gammaprime \geq \gammath$, it holds that $R R^\top \betarobpi$ = 0, and thus the worst-case robust risk of the worst-case robust predictor equals 
\begin{equation*}
    \Lossrobpi(\betarobpi; \Invset, \Mtest) = \gammaprime \Cker^2 +  \noisecovyyS - o(\gamma) = \gammaprime \Cker^2 + c_3(\gamma),
\end{equation*}
where $c_3(\gamma) \leq \noisecovyyS$.
In total, we observe that the worst-case robust risk of \emph{all} considered prediction models grows linearly with the unseen shift strength $\gammaprime$, albeit with different rates. The terms $\| R R^\top \betaa \|$ and $\| R R^\top \betaOLS \|$ can be particularly large, for instance, when there is strong confounding aligned with the unseen shift directions which causes the empirical risk minimizer (OLS) to have a strong signal in these directions. The worst-case robust predictor $\betarobpi$, however, abstains in these directions, thus achieving a smaller rate.

\end{document}